\documentclass[twoside,11pt]{article} 
\usepackage[margin=1.0in]{geometry}

\usepackage{amssymb,amsmath,amsthm}
\usepackage{mathtools}
\usepackage{bm,lscape}
\usepackage{bbm}
\usepackage{mathrsfs,dsfont}
\usepackage{wasysym}
\RequirePackage[
    colorlinks=true,
    linkcolor=blue,
    urlcolor=blue,
    citecolor=blue]{hyperref}
\RequirePackage{natbib}
\usepackage[]{algorithm2e}
\usepackage[utf8]{inputenc}
\usepackage[english]{babel}

\usepackage{graphicx}
\usepackage{color}
\usepackage{rotating}
\usepackage{authblk}
\usepackage{appendix}

\allowdisplaybreaks

\def\bSig\mathbf{\Sigma}

\newcommand{\E}{\mathbb{E}}

\renewcommand{\P}{\mathbb{P}}

\newcommand{\ddr}{\mathrm{d}}

\newcommand{\wz}{w^{(0)}}

\usepackage{booktabs}

\newtheorem{theorem}{Theorem}

\newtheorem{lem}[theorem]{Lemma}

\providecommand{\keywords}[1]
{
  \small	
  \textbf{\textit{Keywords:}} #1
}

\begin{document}

\title{Large-width asymptotics for ReLU neural networks with $\alpha$-Stable initializations}


\author[1]{Stefano Favaro\thanks{stefano.favaro@unito.it}}
\author[2]{Sandra Fortini\thanks{sandra.fortini@unibocconi.it}}
\author[3]{Stefano Peluchetti\thanks{speluchetti@cogent.co.jp}}
\affil[1]{\small{Department of Economics and Statistics, University of Torino and Collegio Carlo Alberto, Italy}}
\affil[2]{\small{Department of Decision Sciences, Bocconi University, Italy}}
\affil[3]{\small{Cogent Labs, Tokyo, Japan}}

\maketitle

\begin{abstract}
There is a recent and growing literature on large-width asymptotic properties of Gaussian neural networks (NNs), namely NNs whose weights are initialized according to Gaussian distributions. In such a context, two popular problems are: i) the study of the large-width distributions of NNs, which characterizes the infinitely wide limit of a rescaled NN in terms of a Gaussian stochastic process; ii) the study of the large-width training dynamics of NNs, which characterizes the infinitely wide dynamics in terms of a deterministic kernel, referred to as the neural tangent kernel (NTK), and shows that, for a sufficiently large width, the gradient descent achieves zero training error at a linear rate. In this paper, we consider these problems for $\alpha$-Stable NNs, namely NNs whose weights are initialized according to $\alpha$-Stable distributions with $\alpha\in(0,2]$, i.e. distributions with heavy-tails. First, for $\alpha$-Stable NNs with a ReLU activation function, we show that if the NN's width goes to infinity then a rescaled NN converges weakly to an $\alpha$-Stable stochastic process. As a difference with respect to the Gaussian setting, our result shows that the choice of the activation function affects the scaling of the NN, that is: to achieve the infinitely wide $\alpha$-Stable process, the ReLU activation requires an additional logarithmic term in the scaling with respect to sub-linear activations. Then, we study the large-width training dynamics of $\alpha$-Stable ReLU-NNs, characterizing the infinitely wide dynamics in terms of a random kernel, referred to as the $\alpha$-Stable NTK, and showing that, for a sufficiently large width, the gradient descent achieves zero training error at a linear rate. The randomness of the $\alpha$-Stable NTK is a further difference with respect to the Gaussian setting, that is: within the $\alpha$-Stable setting, the randomness of the NN at initialization does not vanish in the large-width regime of the training. An extension of our results to deep $\alpha$-Stable NNs is discussed. 
\end{abstract}

\keywords{$\alpha$-Stable stochastic process; gradient descent; infinitely wide limit; large-width training dynamics; neural network; neural tangent kernel; ReLU activation function}


\section{Introduction}

There is a recent and growing literature on large-width asymptotic properties of Gaussian neural networks (NNs), namely NNs whose weights are initialized or distributed according to Gaussian distributions \citep{Nea(96),Wil(97),Der(06),Gar(18),Jac(18),Lee(18),Mat(18),Nov(18),Aro(19),Lee(19),Yan(19),Yan(19a),Yan(19b),Bra(21),Eld(21),Klu(21),Yan1(21),Yan(21),Bas(22)}. Consider the following setting: i) for $d,k\geq1$ let $X$ be the $d\times k$ NN's input, with $x_j=(x_{j1},\ldots,x_{jd})^{T}$ being the $j$-th input (column vector); ii) let $\phi$ be an activation function or nonlinearity; iii) for $m\geq1$ let $W=(w^{(0)}_{1},\ldots,w^{(0)}_{m},w)$ be the NN's weights, such that $w^{(0)}_{i}=(w^{(0)}_{i1},\ldots,w^{(0)}_{id})$ and $w=(w_{1},\ldots,w_{m})$ with the $w^{(0)}_{ij}$'s and the $w_{i}$'s being i.i.d. as a Gaussian distribution with zero mean and variance $\sigma^{2}$. If
 \begin{displaymath}
f_{m}(x_{j})=\sum_{i=1}^{m}w_{i}\phi(\langle w_{i}^{(0)},x_{j}\rangle)
\end{displaymath}
for $j=1,\ldots,k$, then $f_{m}(X)=(f_{m}(x_{1}),\ldots,f_{m}(x_{k}))$ defines a (fully connected feed-forward) Gaussian $\phi$-NN of width $m$. In his seminal work, \cite{Nea(96)} first investigated large-width distributions of Gaussian $\phi$-NNs, characterizing the infinitely wide limit of the NN. In particular, under suitable assumptions on $\phi$, \cite{Nea(96)} showed that an application of the central limit theorem (CLT) leads to the following: if $m\rightarrow+\infty$ then the rescaled NN $m^{-1/2}f_{m}(X)$ converges weakly to a Gaussian stochastic process with covariance function $\Sigma_{X,\phi}$ such that $\Sigma_{X,\phi}[r,s]=\sigma^{2}\E[\phi(\langle w_{i}^{(0)},x_{r}\rangle\phi(\langle w_{i}^{(0)},x_{s}\rangle]$. Some extensions of this result have been obtained for deep NNs \citep{Mat(18)}, general NN's architectures such as convolutional NNs \citep{Yan(19a),Yan(19b)} and infinite-dimensional inputs  \citep{Bra(21),Eld(21)}.

Recent works have investigated the large-width training dynamics of Gaussian NNs, with the training being performed through the  gradient descent \citep{Jac(18),Aro(19),Du(19),Lee(19)}. Let $(X,Y)$ be the training set, where $Y=(y_{1},\ldots,y_{k})$ is the (training) output, with $y_{j}$ being the (training) output for the $j$-th input $x_{j}$. If $f_{m}(X)$ is a Gaussian $\phi$-NN, with  $\phi$ to be the ReLU activation function, then we denote by 
\begin{displaymath}
\tilde{f}_{m}(W,X)=\frac{1}{m^{1/2}}f_{m}(X)
\end{displaymath}
the rescaled (model) output. Starting at random initialization $W(0)$ for the NN's weights, and assuming the squared-error loss function, the gradient flow of $W(t)$ leads to the the training dynamics of $\tilde{f}_{m}(W(t),X)$, that is for $t\geq0$
\begin{equation}\label{ntk_gauss}
\frac{\ddr \tilde{f}_{m}(W(t),X)}{\ddr t}=-(\tilde{f}_{m}(W(t),X)-Y)\eta_{m} H_{m}(W(t),X),
\end{equation}
where $\eta_{m}>0$ is the (continuous) learning rate, and $H_{m}(W(t),X)$ is a $k\times k$ matrix whose $(j,j^{\prime})$ entry is $\langle\partial \tilde{f}_{m}(W(t),x_{j})/\partial W,\partial \tilde{f}_{m}(W(t),x_{j^{\prime}})/\partial W\rangle$. \citet{Du(19)} showed that if $\eta_m=1$, then: i) the kernel $ H_{m}(W(0),X)$ converges in probability, as $m\rightarrow+\infty$, to a deterministic kernel $H^{\ast}(X,X)$, which is referred to as the neural tangent kernel (NTK); ii) the least eigenvalue of $H^{\ast}(X,X)$ is bounded from below by a positive constant $\lambda_0$; iii) for $m$ sufficiently large, the gradient descent achieves zero training error at a linear rate, i.e.
\begin{displaymath}
\Vert Y-\tilde f_m(W(t),X)\Vert_2^2\leq \exp(-\lambda_0t)\Vert Y-\tilde f_m(W(0),X)\Vert_2^2,
\end{displaymath}
with high probability. See  \citet{Aro(19)}, \citet{Yan(19)} and \citet{Yan(21)} for some extensions of these results to deep NNs and general architectures.

\subsection{Our contributions}

We study large-width asymptotic properties of $\alpha$-Stable ReLU-NNs, namely NNs with a ReLU activation function and weights initialized according to $\alpha$-Stable distributions \citep{Sam(94)}. For $\alpha\in(0,2]$, $\alpha$-Stable distributions form a class of heavy tails distributions, with $\alpha=2$ being the Gaussian distribution. \citet{Nea(96)} first considered $\alpha$-Stable distributions to initialize NNs, showing that while all Gaussian weights vanish in the infinitely wide limit, some $\alpha$-Stable weights retain a non-negligible contribution, allowing to represent ``hidden features" \citep{Der(06),For(20),Lee(22)}. Such a behaviour is attributed to the diversity of the NN's path properties as $\alpha\in(0,2]$ varies, as shown in Figure \ref{fig:shapes}, which makes $\alpha$-Stable NNs more flexible than Gaussian NNs. Motivated by these works, \citet{Fav(20),Fav(21)} provided the following result for an $\alpha$-Stable $\phi$-NN $f_{m}(X;\alpha)$: for $\alpha\in(0,2)$ and a sub-linear $\phi$, if $m\rightarrow+\infty$ then the rescaled NN $m^{-1/\alpha}f_{m}(X;\alpha)$ converges weakly to an $\alpha$-Stable stochastic process, that is a process with $\alpha$-Stable finite-dimensional distributions. Here, we extend this result to the ReLU activation, which is one of the most popular linear activation function. In particular, we show that if $m\rightarrow+\infty$, then the $\alpha$-Stable ReLU-NN $(m\log m)^{-1/\alpha}f_{m}(X;\alpha)$ converges weakly to an $\alpha$-Stable process. While for NNs with a single input, i.e. $k=1$, such a result follows by an application of the   generalized CLT for heavy tails distributions \citep{Uch(11),Bor(22)}, for $k>1$ the generalized CLT does not apply, leading us to develop an alternative proof that may be of independent interest for multidimensional $\alpha$-Stable distributions. It turns out that in the $\alpha$-Stable setting, differently from the Gaussian setting, the choice of $\phi$ affects the scaling of the NN, that is: to achieve the infinitely wide $\alpha$-Stable process, the use of the ReLU activation in place of a sub-linear activation results in a change of the scaling $m^{-1/\alpha}$ of the NN through the additional $(\log m)^{-1/\alpha}$ term.

\begin{figure}[htp!]
\centering
\includegraphics[width=0.28\textwidth, trim = 40 40 40 40]{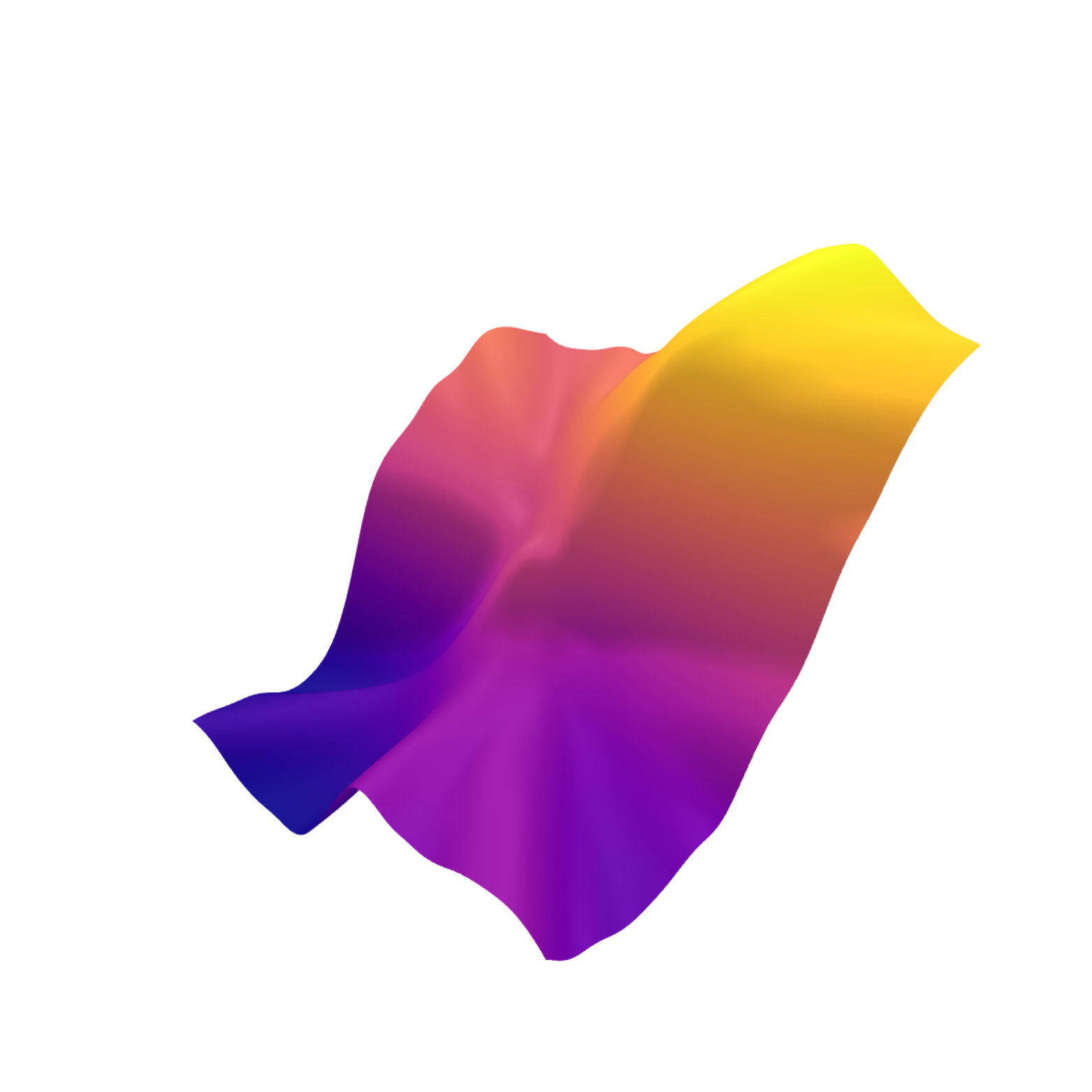}
\includegraphics[width=0.28\textwidth, trim = 40 40 40 40]{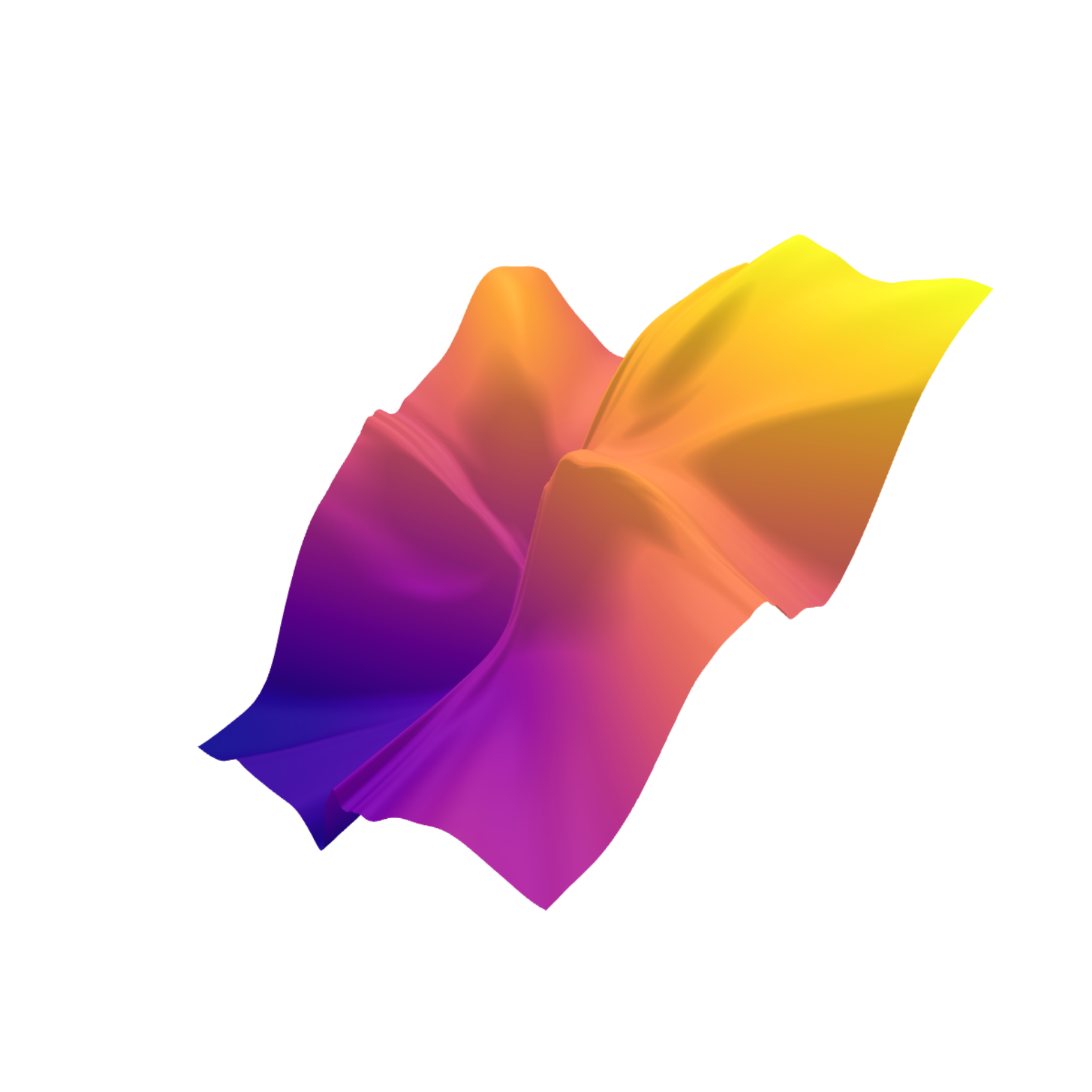}\\
\includegraphics[width=0.28\textwidth, trim = 40 40 40 40]{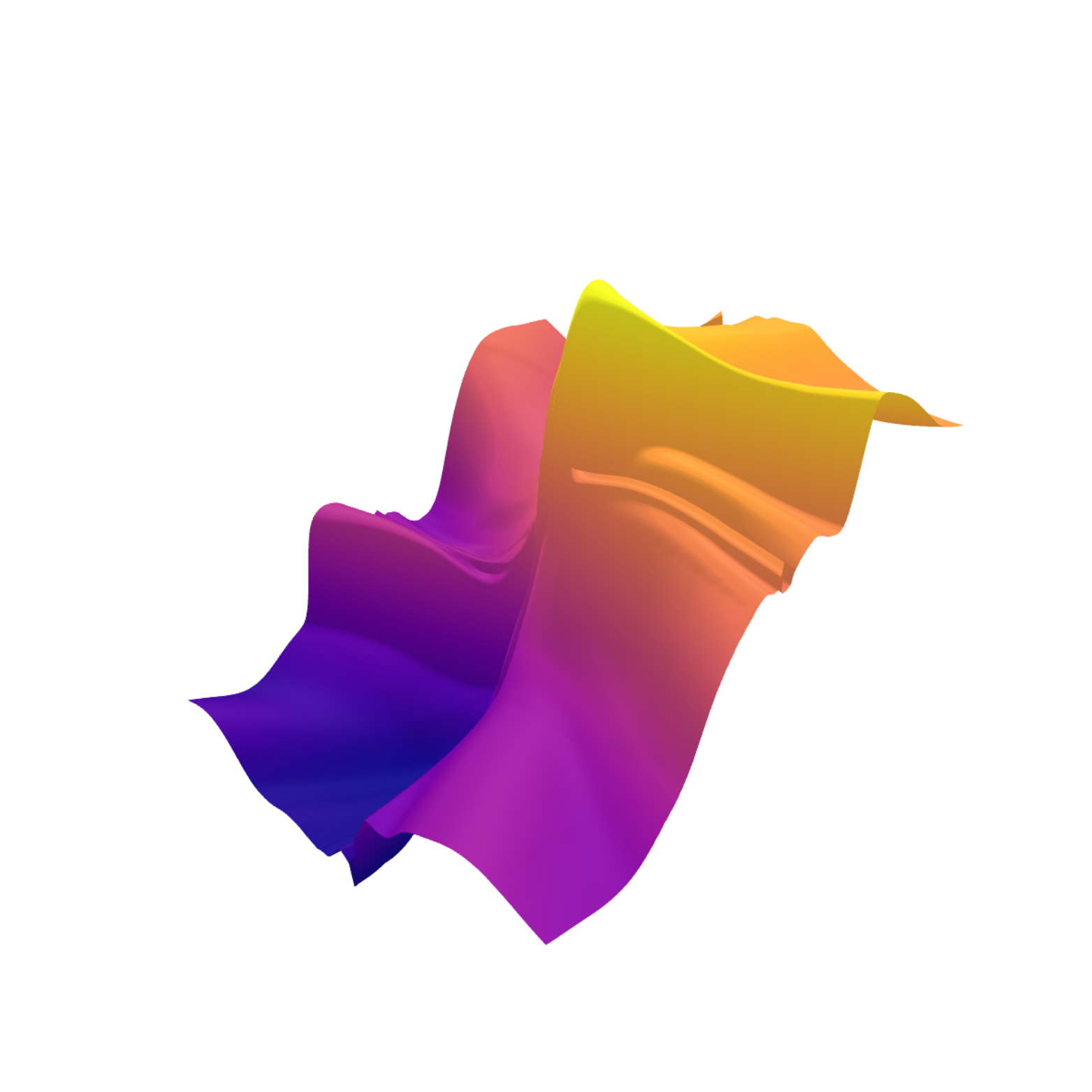}
\includegraphics[width=0.28\textwidth, trim = 40 40 40 40]{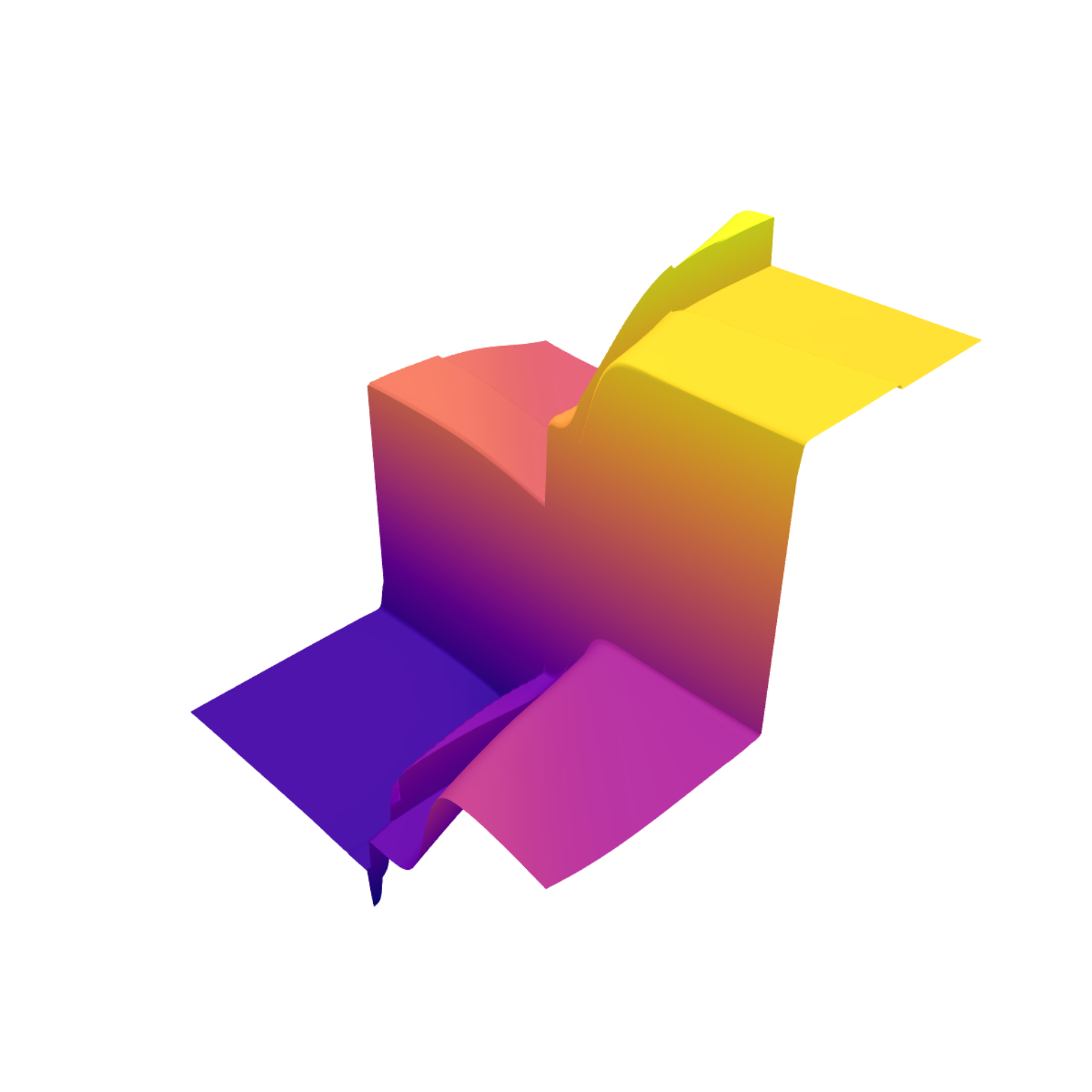}
\caption{Sample paths of shallow $\alpha$-Stable NNs mapping $[0,1]^2$ to $\mathbb{R}$, with a $\tanh$ activation function and witdth $m=1024$: i) $\alpha=2.0$ (Gaussian distribution) top-left; ii) $\alpha=1.5$ top-right; iii) $\alpha=1.0$ (Cauchy distribution) bottom-left; iv) $\alpha=0.5$ (L\'evy distribution) bottom-right.}\label{fig:shapes}
\end{figure}

Then, our main contribution is the study of the large-width training dynamics of $\alpha$-Stable ReLU-NNs, thus generalizing to the $\alpha$-Stable setting the main result of \citet{Du(19)}, as well as some results of \citet{Jac(18)} and \citet{Aro(19)}. For $\alpha\in(0,2)$ and a training set $(X,Y)$, we denote by 
\begin{displaymath}
\tilde{f}_{m}(W,X;\alpha)=\frac{1}{(m\log m)^{1/\alpha}}f_{m}(X;\alpha)
\end{displaymath}
the rescaled (model) output, and we consider the training of the NN performed through gradient descent under the squared-error loss function. By writing the training dynamics of $\tilde{f}_{m}(W(t),X;\alpha)$ as in \eqref{ntk_gauss}, with $\eta_{m}$ being the (continuous) learning rate and $H_{m}(W(t),X)$ being the kernel in the $\alpha$-Stable setting, we show that if $\eta_{m}=(\log m)^{2/\alpha}$ then: i) the rescaled kernel $(\log m)^{2/\alpha} H_{m}(W(0),X)$ converges in distribution, as $m\rightarrow+\infty$, to an $(\alpha/2)$-Stable (almost surely) positive definite random kernel $\tilde{H}^{\ast}(X,X;\alpha)$, which is referred to as the $\alpha$-Stable NTK; ii) during training $t>0$, for every $\delta>0$ the least eigenvalue of $(\log m)^{2/\alpha}\tilde H_m(W(t),X;\alpha)$ remains bounded away from zero, for $m$ sufficiently large, with probability $1-\delta$; iii) for every $\delta>0$ the gradient descent achieves zero training loss at a linear rate, for $m$ sufficiently large, with probability $1-\delta$. The randomness of the $\alpha$-Stable NTK is a further difference with respect to the Gaussian setting, and it makes the convergence analysis of the gradient descent more challenging than in the Gaussian setting. Our work is the first to investigate the the large-width training dynamics of NNs with weights initialized through heavy tails distributions, and it shows that, within the $\alpha$-Stable setting, the randomness of the NN at initialization does not vanish in the large-width regime of the training. Such a peculiar behaviour may be viewed as the counterpart, at the training level, of the large-width behaviour described in the work of \citet{Nea(96)}.

\subsection{Organization of the paper}

The paper is organized as follows. In Section \ref{sec2} we the study of the large-width distributions of $\alpha$-Stable ReLU-NNs, characterizing the infinitely wide limit of a rescaled NN in terms of an $\alpha$-Stable process. In Section \ref{sec3} we study the large-width training dynamics of $\alpha$-Stable ReLU-NNs, characterizing the infinitely wide dynamics in terms of the $\alpha$-Stable NTK, and showing that, for a sufficiently large width, the gradient descent achieves zero training error at a linear rate, with high probability. Section \ref{sec4} contains a discussion of our results, their extension to deep $\alpha$-Stable NNs, and some directions for future work. Appendices contain the proofs of our results and a brief review of multidimensional $\alpha$-Stable distributions.


\section{$\alpha$-Stable ReLU-NNs and their large-width distribution}\label{sec2}

It is useful to recall the definition of the multidimensional $\alpha$-Stable distribution. See \citet[Chapter 1 and Chapter 2]{Sam(94)}. For $\alpha\in(0,2]$, a random variable $S\in\mathbb{R}$ is distributed as a symmetric and centered $1$-dimensional $\alpha$-Stable distribution with scale $\sigma>0$ if its characteristic function is
\begin{displaymath}
\E(\exp\{\text{i}zS\})=\exp\left\{-\sigma^{\alpha}|z|^{\alpha}\right\},
\end{displaymath}
and we write $S\sim\text{St}(\alpha,\sigma)$. The parameter $\alpha$ is referred to as the stability. In particular, if $\alpha=2$ then $S$ is distributed according to a Gaussian distribution with mean $0$ and variance $\sigma^{2}$. Let $\mathbb{S}^{k-1}$ be the unit sphere in $\mathbb{R}^{k}$, with $k\geq1$, and let $\Gamma$ be a symmetric finite measure on $\mathbb S^{k-1}$. For $\alpha\in(0,2]$, random variable $S\in\mathbb{R}^{k}$ is distributed as a symmetric and centered $k$-dimensional $\alpha$-Stable distribution with spectral measure $\Gamma$ if its characteristic function is
\begin{displaymath}
\E(\exp\{\text{i}\langle z,S\rangle\})=\exp\left\{-\int_{\mathbb{S}^{k-1}}|\langle z,s \rangle|^{\alpha}\Gamma(\ddr s)\right\},
\end{displaymath}
and we write $S\sim\text{St}_{k}(\alpha,\Gamma)$. Let $1_{r}$ be the $r$-dimensional (column) vector with $1$ in the $r$-th entry and $0$ elsewhere, for any $r=1,\ldots,k$. Then, the  $r$-th element of $S$, that is $S1_{r}$ is distributed as an $\alpha$-Stable distribution with scale
\begin{displaymath}
\sigma=\left(\int_{\mathbb{S}^{k-1}}|\langle 1_{r},s\rangle|^{\alpha}\Gamma(\ddr s)\right)^{1/\alpha}.
\end{displaymath}
We deal mostly with $k$-dimensional  $\alpha$-Stable distributions with discrete spectral measure, that is a measure $\Gamma(\cdot)=\sum_{1\leq i\leq n}\gamma_{i}\delta_{s_{i}}(\cdot)$ with $n\in\mathbb{N}$, $\gamma_{i}\in\mathbb{R}$ and $s_{i}\in\mathbb{S}^{k-1}$, for $i=1,\ldots,n$ \citep[Chapter 2]{Sam(94)}. Throughout this paper, it is assumed that all the random variables are defined on a common probability space, say $(\Omega,\mathcal{F},\P)$, unless otherwise stated.

To define an $\alpha$-Stable ReLU-NNs, consider the following setting: i) for any $d,k\geq1$ let $X$ be the $d\times k$ NN's input, with $x_j=(x_{j1},\ldots,x_{jd})^{T}$ being the $j$-th input (column vector); ii) for $m\geq1$ let $W=(w^{(0)}_{1},\ldots,w^{(0)}_{m},w)$ be the NN's weights, such that $w^{(0)}_{i}=(w^{(0)}_{i1},\ldots,w^{(0)}_{id})$ and $w=(w_{1},\ldots,w_{m})$.  If
\begin{displaymath}
	f_{m}(W,x_{j};\alpha)=\sum_{i=1}^{m}w_{i}\langle w_{i}^{(0)},x_{j}\rangle I(\langle w_{i}^{(0)},x_{j}\rangle>0)
\end{displaymath}
for $j=1,\ldots,k$, with $I(\cdot)$ being the indicator function, then $f_{m}(W,X;\alpha)=(f_{m}(W,x_{1};\alpha),\ldots,f_{m}(W,x_{k};\alpha))$ defines a ReLU-NN of width $m$. Now, let $W(0)=(w^{(0)}_{1}(0),\ldots,w^{(0)}_{m}(0),w(0))$ be the NN weights at random initialization. If the weight  $w^{(0)}_{ij}$'s and $w_{i}$'s are initialized as  i.i.d.  $\alpha$-Stable random variables, with $\alpha\in(0,2)$ and $\sigma>0$, then $f_{m}(W(0),X;\alpha)$ defines an $\alpha$-Stable ReLU-NN of width $m$. Without loss of generality we assume $\sigma=1$. The case $\alpha=2$, which corresponds to the Gaussian setting, is excluded by our analysis, though some of our results are valid also for $\alpha=2$. The next theorem characterizes the infinitely wide limit of $\alpha$-Stable ReLU-NNs. We denote by $Z_m\stackrel{w}{\longrightarrow} Z$ the weak convergence, as $m\rightarrow+\infty$, of the  sequence of random vectors $(Z_m)_{n\geq1}$ to the random vector $Z$. Moreover, for $\alpha\in (0,1)$ let
 \begin{equation}\label{const_stabl}
 	C_\alpha=\left\{\begin{array}{ll}
 		\frac{1-\alpha}{\Gamma(2-\alpha)\cos (\pi\alpha/2)}&\alpha\neq 1\\[0.2cm]
 		\frac{2}{\pi}&\alpha=1.
 	\end{array}
 	\right.
 \end{equation}
We refer to \citet[Chapter 1 and Chapter 2]{Sam(94)} for further details on $C_{\alpha}$ in the context of multidimensional $\alpha$-Stable distributions.

\begin{theorem}\label{teo_priorlimit}
Let $f_{m}(W(0),X;\alpha)$ be an $\alpha$-Stable ReLU-NN. If $m\rightarrow+\infty$ then
\begin{displaymath}
\frac{1}{(m\log m)^{1/\alpha}}f_{m}(W(0),X;\alpha)\stackrel{\text{w}}{\longrightarrow}f(X),
\end{displaymath}
where $f(X)\sim\text{St}_{k}(\alpha,\Gamma_{X})$, with the spectral measure $\Gamma_{X}$ being of the following form:
\begin{displaymath}
\Gamma_{X}=\frac{C_\alpha}{4}\sum_{i=1}^{d} (\Vert[x_{ji}I(x_{ji}>0)]_j\Vert^\alpha)D_{i}^{+}(X)+\Vert
[x_{ji}I(x_{ji}<0)]_j\Vert^\alpha)D_{i}^{-}(X)
\end{displaymath}
such that
\begin{displaymath}
D_{i}^{+}(X)=
\delta\left(\dfrac{[x_{ji}I(x_{ji}>0)]_j}{\Vert
	[x_{ji}I(x_{ji}>0)]_j\Vert}\right)
+\delta\left(-\dfrac{[x_{ji}I(x_{ji}>0)]_j}{\Vert
	[x_{ji}I(x_{ji}>0)]_j\Vert}\right)
\end{displaymath}
and
\begin{displaymath}
D_{i}^{-}(X)=
\delta\left(\dfrac{[x_{ji}I(x_{ji}<0)]_j}{\Vert
	[x_{ji}I(x_{ji}<0)]_j\Vert}\right)
+\delta\left(-\dfrac{[x_{ji}I(x_{ji}<0)]_j}{\Vert
	[x_{ji}I(x_{ji}<0)]_j\Vert}\right),
\end{displaymath}
where, for any $s\in \mathbb S^{k-1}$, $\delta(s)$ is probability measure degenerate in $s$, and $C_\alpha$ is a constant defined in \eqref{const_stabl}. The stochastic process $f(X)=(f(x_{1}),\ldots,f(x_{k}))$, as a process indexed by the NN's input $X$, is an $\alpha$-Stable process with spectral measure $\Gamma_{X}$.
\end{theorem}

See Appendix \ref{proofTh1} for the proof of Theorem \ref{teo_priorlimit}. For a broad class of bounded or sub-linear activation functions, the main result of \citet{Fav(21)} characterizes the large-width distribution of deep $\alpha$-Stable NNs. In particular, let
\begin{displaymath}
f_{m}(x_{j};\alpha)=\sum_{i=1}^{m}w_{i}\phi\langle w_{i}^{(0)},x_{j}\rangle
\end{displaymath}
be the $\alpha$-Stable NN of width $m$ for the input $x_{j}$, for $j=1,\ldots,k$, with the function $\phi$ being a bounded  activation function. Now, let $f_{m}(X;\alpha)=(f_{m}(x_{1};\alpha),\ldots,f_{m}(x_{k};\alpha))$. According to \citet[Theorem 1.2]{Fav(21)}, if $m\rightarrow+\infty$ then
\begin{equation}\label{lim_bound}
\frac{1}{m^{1/\alpha}}f_{m}(X;\alpha)\stackrel{\text{w}}{\longrightarrow}f(X),
\end{equation}
with $f(X)$ being an $\alpha$-Stable process with spectral measure $\Gamma_{X,\phi}$. Theorem \ref{teo_priorlimit} provides an extension of \citet[Theorem 1.2]{Fav(21)} to the ReLU activation function, which is one of the most popular unbounded activation function. It is useful to discuss Theorem \ref{teo_priorlimit} with respect to the scaling $(m\log m)^{-1/\alpha}$, which is required to achieve the infinitely wide $\alpha$-Stable process. In particular, Theorem \ref{teo_priorlimit} shows that the use of the ReLU activation in place of a bounded activation results in a change of the scaling $m^{-1/\alpha}$ in \eqref{lim_bound}, through the inclusion of the $(\log m)^{-1/\alpha}$ term. This is a critical difference between the $\alpha$-Stable setting and Gaussian setting, as in the latter the choice of the activation function $\phi$ does not affect the scaling $m^{-1/2}$ required to achieve the infinitely wide Gaussian process. For $k=1$, we refer to \citet{Bor(22)} for a detailed analysis of infinitely wide limits of $\alpha$-Stable NNs with general classes of sub-linear, linear and super-linear activation functions.


\section{Large-width training dynamics of $\alpha$-Stable ReLU-NNs}\label{sec3}

Let $f_{m}(W,X;\alpha)$ be an $\alpha$-Stable ReLU-NN, and let $(X,Y)$ be the training set, where $Y=(y_{1},\ldots,y_{k})$ is the (training) output, with $y_{j}$ being the (training) output for the $j$-th input $x_{j}$. We consider the rescaled (model) output of the form
\begin{displaymath}
	\tilde{f}_{m}(W,X;\alpha)=\frac{1}{(m\log m)^{1/\alpha}}f_{m}(W,X;\alpha),
\end{displaymath}
and denote by $\tilde{f}_{m}(W,x_{j};\alpha)=(m\log m)^{-1/\alpha}f_{m}(W,x_{j};\alpha)$ the (model) output of $x_{j}$, for $j=1,\ldots, k$. Then, by assuming the squared-error loss function $\ell(y_{j},\tilde f_m(W,x_j;\alpha))=2^{-1}\sum_{1\leq j\leq k}(\tilde f_m(W,x_j;\alpha)-y_j)^2$, a direct application of the chain rule leads to the NN's training dynamics. That is for any $t\geq0$ we write
\begin{equation}\label{dynamic_stable}
	\frac{\ddr \tilde{f}_{m}(W(t),X;\alpha)}{\ddr t}=-(\tilde{f}_{m}(W(t),X;\alpha)-Y)\eta_{m} H_{m}(W(t),X),
\end{equation}
where the kernel $H_{m}(W(t),X)$ is a $k\times k$ matrix whose $(j,j^{\prime})$ entry is of the form
\begin{equation}\label{kernel_ntk}
	H_{m}(W(t),X)[j,j^{\prime}]=\left\langle\frac{\partial \tilde{f}_{m}(W(t),x_{j};\alpha)}{\partial W},\frac{\partial \tilde{f}_{m}(W(t),x_{j^{\prime}};\alpha)}{\partial W}\right\rangle,
\end{equation} 
and $\eta_{m}$ is the (continuous) learning rate. We show that if $\eta_{m}=(\log m)^{2/\alpha}$ then: i) the rescaled kernel at initialization $\tilde{H}_{m}(W(0),X)=\eta_m H_{m}(W(0),X)$ converges in distribution to an $(\alpha/2)$-Stable (almost surely) positive definite random kernel $\tilde{H}^{\ast}(X,X;\alpha)$, as $m\rightarrow\infty$; ii) during training $t>0$, for every $\delta>0$ the least eigenvalue of the kernel $\tilde H_m(W(t),X)$ remains bounded away from zero, for $m$ sufficiently large, with probability $1-\delta$; iii) for every $\delta>0$ the gradient descent achieves zero training loss at a linear rate, with probability $1-\delta$. We denote by $\lambda_{\text{min}}(\cdot )$, $\Vert \cdot\Vert_F$ and $\Vert \cdot \Vert_2$ the minimum eigenvalue, the Frobenius and operator norms of symmetric and positive semi-definite matrices.

\subsection{Infinitely wide limits of $\tilde{H}_{m}(W(0),X)$}\label{sec:kernel0}
For $\alpha$-Stable ReLU-NNs, we study the large-width behaviour of the kernel ${H}_{m}(W(0),X)$ in  \eqref{kernel_ntk}. In particular, if
\begin{equation}\label{eq:kernel0}
\tilde{H}_{m}(W,X)=(\log m)^{2/\alpha}H_{m}(W,X),
\end{equation}
then $\tilde H_m(W(0),X)$ converges in distribution, as $m\rightarrow\infty$, to a positive definite random matrix $\tilde H^\ast(X,X,\alpha)$, with $(\alpha/2)$-stable distribution. This result allows to prove that the minimum eigenvalue of  $\tilde{H}_{m}(W(0),X)$ is bounded away from zero, with arbitrarily high probability, for $m$ sufficiently large. Critical for these results is the fact that $\tilde{H}_{m}(W,X)$ can be decomposed as follows:
\begin{align}\label{kernel}
&\tilde{H}_{m}(W,X)=\tilde{H}^{(1)}_{m}(W,X)+\tilde{H}^{(2)}_{m}(W,X),
\end{align}
with $\tilde{H}^{(1)}_{m}(W,X)$ and $\tilde{H}^{(2)}_{m}(W,X)$ being matrices whose $(j,j^{\prime})$ entries are of the form
\begin{align}\label{kernel_1}
&\tilde{H}^{(1)}_{m}(W,X)[j,j']\\
&\notag\quad=\frac{1}{m^{2/\alpha}}\sum_{i=1}^{m}w_{i}^{2}\langle x_{j},x_{j^{\prime}}\rangle I(\langle w_{i}^{(0)},x_{j}\rangle>0)I(\langle w_{i}^{(0)},x_{j^{\prime}}\rangle>0),
\end{align}
and
\begin{align}\label{kernel_2}
&\tilde{H}^{(2)}_{m}(W,X)[j,j']\\
&\notag\quad=\frac{1}{m^{2/\alpha}}\sum_{i=1}^{m}\langle w_{i}^{(0)},x_{j}\rangle I(\langle w_{i}^{(0)},x_{j}\rangle>0)\langle w_{i}^{(0)},x_{j^{\prime}}\rangle I(\langle w_{i}^{(0)},x_{j^{\prime}}\rangle>0).
\end{align}
The next theorem characterizes the infinitely wide limits of $\tilde{H}^{(1)}_{m}(W(0),X)$, $\tilde{H}^{(2)}_{m}(W(0),X)$,  and $\tilde H_m(W(0),X)$, and provides expressions for their spectral measures.

\begin{theorem}\label{grad_asy}
	Let $\tilde{H}_{m}(W,X)$,  $\tilde H_m^{(1)}(W,X)$ and $\tilde H_m^{(2)}(W,X)$ be the matrices defined in \eqref{eq:kernel0}, \eqref{kernel_1}, and \eqref{kernel_2}, respectively. Moreover, for every $u\in \{0,1\}^k$, let 
\begin{displaymath}
	B_u=\{v\in \mathbb R^d:\langle v,x_j\rangle>0\mbox{ if  }u_j=1, \langle v,x_j\rangle\leq 0 \mbox{ if  }u_j=0,j=1,\dots,k \},
\end{displaymath} 
and for every $i=1,\dots,d$, let $e_i$ denote the $d$-dimensional vector satisfying 
\begin{displaymath}
e_{ij}=1 \mbox{ for }j=i, \quad e_{ij}=0 \mbox{ for } j\neq i.
\end{displaymath}
As $m\rightarrow+\infty$, 
	 $$(\tilde H_m^{(1)}(W(0),X),\tilde H_m^{(2)}(W(0),X))\stackrel{w}{\longrightarrow}(\tilde H_1^\ast(\alpha),\tilde H_2^\ast(\alpha)),$$
	 where $\tilde H_1^\ast(\alpha)$ and $\tilde H_2^\ast(\alpha)$ are stochastically independent, positive semi-definite random matrices, distributed as $(\alpha/2)$-Stable distributions with spectral measures
	 	\begin{equation}\label{spectral_1}
		\Gamma^{\ast}_{1}=C_{\alpha/2}\sum_{u\in\{0,1\}^k}
		\mathbb P(w_i^{(0)}(0)\in B_u)\frac{ \delta\left(\frac{
			\left[\langle x_j,x_{j'}\rangle u_ju_{j'}\right]_{j,j'}
		}{(\sum_{j,j'}\langle x_j,x_{j'}\rangle^2u_ju_{j'})^{1/2}}
		\right)}{\left(\sum_{j,j'}\langle x_j,x_{j'}\rangle^2u_ju_{j'}\right)^{-\alpha/4}},
	\end{equation}
 and 
 \begin{equation}
 	\label{spectral2}
 	\Gamma^*_2=C_{\alpha/2}\sum_{u\in \{0,1\}^k}\sum_{\{i:\{e_i,-e_i\}\cap B_u\neq\emptyset\}}\frac{\delta\left(
 	\frac{[x_{ji}u_jx_{j'i}u_{j'}]_{j,j'}}{\sum_jx_{ji}^2u_j}
 	\right)}{\left(\sum_j x_{ji}^2u_j\right)^{-\alpha/2}},
 \end{equation}
 respectively, where $C_{\alpha/2}$ is a constant defined in \eqref{const_stabl}. Furthermore, as $m\rightarrow\infty$, 
 $$\tilde H_m(W(0),X)\stackrel{w}{\longrightarrow} \tilde H^\ast(X,X;\alpha),$$ 
 where $\tilde H^\ast(X,X;\alpha)$ is a positive semi-definite random matrix, distributed according to an $(\alpha/2)$-Stable distribution with spectral measure of the form $\Gamma^{\ast}=\Gamma^\ast_1+\Gamma^\ast_2$. 
\end{theorem}

See Appendix \ref{proofTh2} for the proof of Theorem \ref{grad_asy}. It turns out that the probability distributions of the random matrices $\tilde H_1^{\ast}(\alpha)$ and $\tilde H_2^{\ast}(\alpha)$ are absolutely continuous in suitable subspaces of the space of symmetric and positive semi-definite matrices. In turn, this fact implies that the minimum eigenvalues of $\tilde H_m^{(1)}(W(0),X)$ and of $\tilde H_m^{(2)}(W(0),X)$ are bounded away from zero, uniformly in $m$, for $m$ sufficiently large, with arbitrarily high probability. 

\begin{theorem}\label{th:grad_asy2}
Under the assumptions of Theorem \ref{grad_asy}, for every $\delta>0$ there exist strictly positive numbers $\lambda_0$,  $\lambda_1$ and $\lambda_2$ such that, for $m$ sufficiently large, 
 \begin{displaymath}
 	\lambda_{\text{min}}(\tilde{H}^{(i)}_{m}(W(0),X))>\lambda_{i}\quad \quad i=1,2,
 \end{displaymath} 
 and
 \begin{displaymath}
 	\lambda_{\text{min}}(\tilde{H}_{m}(W(0),X))>\lambda_{0}.
 \end{displaymath} 
 with probability at least $1-\delta$.
\end{theorem}

See Appendix \ref{proofTh3} for the proof of Theorem \ref{th:grad_asy2}. For a Gaussian ReLU-NN,  \citet{Du(19)} showed that if $\eta_m=1$, then: i) the kernel $ H_{m}(W(0),X)$ converges in probability, as $m\rightarrow+\infty$, to the NTK $H^{\ast}(X,X)$, which is a deterministic kernel; ii) the least eigenvalue of $H^{\ast}(X,X)$ is bounded from below by a positive constant $\lambda_0$. See also \citet{Jac(18)}, \citet{Aro(19)} and \citet{Lee(19)}, and references therein, for the study of large-width training dynamics of Gaussian ReLU NNs, and generalizations thereof. Theorem \ref{grad_asy} and Theorem \ref{th:grad_asy2} extend the results of \citet{Du(19)} to the $\alpha$-Stable setting, for $\alpha\in(0,2)$, showing that: i) the rescaled kernel $(\log m)^{2/\alpha} H_{m}(W(0),X)$ converges in distribution, as $m\rightarrow+\infty$, to the $\alpha$-Stable NTK $\tilde{H}^{\ast}(X,X;\alpha)$, which is $(\alpha/2)$-Stable (almost surely) positive definite random kernel; ii) during training $t>0$, for every $\delta>0$ the least eigenvalue of the kernel $\tilde H_m(W(t),X;\alpha)$ remains bounded away from zero, for $m$ sufficiently large, with probability $1-\delta$. The randomness of the $\alpha$-Stable NTK provides a critical difference between the $\alpha$-Stable setting and the Gaussian setting, showing that in the $\alpha$-Stable setting the randomness of the NN at initialization does not vanish in the large-width regime of the training.

\subsection{Large-width training dynamics of $\alpha$-Stable ReLU-NNs}\label{sec:lower_bound}
We exploit Theorem \ref{grad_asy} and Theorem \ref{th:grad_asy2} to study the large-width training dynamics of $\alpha$-Stable NNs. The next theorem shows that, if $m$ is sufficiently large, then with high probability
the minumum eigenvalue of $\tilde{H}_{m}(W(t),X)$
remains bounded away from zero. This 
property is critical for the rate of convergence of the training.

\begin{theorem}\label{teo_ntk0}
	For any $k\geq1$ let the collection of NN's inputs $x_1,\dots,x_k$ be linearly independent, and such that $\Vert x_{j}\Vert=1$. 
	Let $\gamma\in (0,1)$ and $c>0$ be fixed numbers. Let  $\tilde{H}_{m}(W,X)$ and $\tilde H_m^{(2)}(W,X)$ be the random matrices defined as in \eqref{eq:kernel0} and \eqref{kernel_2}, respectively. For every $\delta>0$ the following properties hold true for every $t\geq 0$, with probability at least $1-\delta$, for $m$ sufficiently large:
	\begin{itemize}
		\item[(i)] for every $j=1,\dots,k$, $$\displaystyle
		(\log m)^{2/\alpha}\left\Vert \frac{\partial \tilde f_{m}}{\partial w}(W(t),x_j;\alpha)-\frac{\partial \tilde f_{m}}{\partial w}(W(0),x_j;\alpha)\right\Vert_F^2 < c m^{-2\gamma/\alpha};
		$$
				\item[(ii)] there exists $\lambda_0>0$ such that
		$$
		\Vert \tilde{H}_{m}^{(2)}(W(t),X)-\tilde{H}_{m}^{(2)}(W(0),X)\Vert_F<\lambda_0 m^{-\gamma/\alpha}
		$$
					and
$$
	\lambda_{\text{min}}(\tilde{H}_{m}(W(t),X))>\frac{\lambda_{0}}{2}.
$$
	\end{itemize}
	\end{theorem}
	
See Appendix \ref{proofTh4} for the proof of Theorem \ref{teo_ntk0}. Theorem \ref{teo_ntk0} is critical to complete our study on the large-width training dynamics of $\alpha$-Stable ReLU-NNs. In particular, from Theorem \ref{teo_ntk0}, for a fixed $\delta>0$, let $m$ and $\lambda_{0}>0$ be such that
\begin{displaymath}
	\lambda_{\text{min}}(\tilde{H}_{m}(W(s),X))>\frac{\lambda_{0}}{2}.
\end{displaymath} 
 for every $s\leq t$, on a set $N\in\mathcal F$ with $\mathbb P[N]>1-\delta$. Then, for a random initialization $W(0)(\omega)$ of the $\alpha$-Stable ReLU-NN , with $\omega\in N$, it holds true that
\begin{align*}
	&\frac{\ddr}{\ddr s}\Vert Y-\tilde{f}_{m}(W(s)(\omega),X;\alpha)\Vert_{2}^{2}\leq-\lambda_0\Vert Y-\tilde{f}_{m}(W(s)(\omega),X;\alpha)\Vert_{2}^{2},
\end{align*} 
and hence
\begin{displaymath}
	\frac{\ddr}{\ddr s}\exp(\lambda_0 s)\Vert Y-\tilde{f}_{m}(W(s)(\omega),X;\alpha)\Vert_{2}^{2}\leq 0.
\end{displaymath}
Since $\exp(\lambda_0s)\Vert Y-\tilde{f}_{m}(W(s)(\omega),X;\alpha)\Vert_{2}^{2}$ is a decreasing function of $s>0$, then we write
\begin{displaymath}
	\Vert Y-\tilde{f}_{m}(W(s)(\omega),X;\alpha)\Vert_{2}^{2}\leq \exp(-\lambda_0s)\Vert Y-\tilde{f}_{m}(W(0)(\omega),X;\alpha)\Vert_{2}^{2}.
\end{displaymath}
The next theorem summarizes the main result of this section, completing our study.

\begin{theorem}\label{teo_ntk}
	For any $k\geq1$ let the collection of NN's inputs $x_1,\dots,x_k$ be linearly independent, and such that $\Vert x_{j}\Vert=1$. Under the dynamics \eqref{dynamic_stable}, if $\eta_{m}=(\log m)^{2/\alpha}$ then for every $\delta>0$ there exists $\lambda_{0}>0$ such that, for $m$ sufficiently large and any $t>0$, with probability at least $1-\delta$ it holds true that
	\begin{displaymath}
		\Vert Y-\tilde{f}_{m}(W(t),X;\alpha)\Vert_{2}^{2}\leq\exp(-\lambda_0 t)\Vert Y-\tilde{f}_{m}(W(0),X;\alpha)\Vert_{2}^{2}.
	\end{displaymath}
\end{theorem}


\section{Discussion}\label{sec4}

In this paper, we investigated large-width asymptotic properties of shallow $\alpha$-Stable ReLU-NNs, focusing on two popular problems: i) the study of the large-width distribution of the NN; ii) the study of the large-width training dynamics of the NN. With regards to the large-width distribution of the NN, we showed that, as the NN's width goes to infinity, a rescaled $\alpha$-Stable ReLU-NN converges weakly to an $\alpha$-Stable process. As a novelty with respect to the Gaussian setting, it turns out that in the $\alpha$-Stable setting the choice of the activation function affects the scaling of the NN, that is: to achieve the infinitely wide $\alpha$-Stable process, the ReLU activation requires an additional logarithmic term in the scaling with respect to sub-linear activations. With regards to the large-width training dynamics of the NN, we characterized the infinitely wide dynamics in terms of the $\alpha$-Stable NTK, and we showed that, for a sufficiently large width, the gradient descent achieves zero training error at a linear rate. The randomness of the $\alpha$-Stable NTK is a further novelty with respect to the Gaussian setting, that is: within the $\alpha$-Stable setting, the randomness of the NN at initialization does not vanish in the large-width regime of the training. Our work extends the main result of \citet{Fav(20),Fav(21)} to the popular ReLU activation function, and then presents the first analysis of the large-width training dynamics of NNs in the $\alpha$-Stable setting, thus generalizing to heavy-tails distributions the main result of \citet{Du(19)}, as well as some results of \citet{Jac(18)} and \citet{Aro(19)}. The use of the $\alpha$-Stable distributions to initialize NNs, in place of Gaussian distributions, brought some interesting phenomena, paving the way to fruitful directions for future research.

It remains open to establish a large-width equivalence between training an $\alpha$-Stable ReLU-NN and performing a kernel regression with the $\alpha$-Stable NTK. \citet{Jac(18)} showed that for Gaussian NNs, during training $t>0$, if $m$ is sufficiently large then the fluctuations of the squared Frobenious norm $\Vert H_{m}(W(t),X)-H_{m}(W(0),X)\Vert^{2}_{F}$ are vanishing. This suggested to replace $\eta_{m}H_{m}(W(t),X)$ with the NTK $H^{\ast}(X,X)$ in the dynamics \eqref{ntk_gauss}, and write
\begin{displaymath}
\frac{\ddr f^{\ast}(t,X)}{\ddr t}=-(f^{\ast}(t,X)-Y) H^{\ast}(X,X).
\end{displaymath}
This is precisely the dynamics of a kernel regression under gradient flow, for which at $t\rightarrow+\infty$ the prediction for a generic test point $x\in\mathbb{R}^{d}$ is of the form $f^{\ast}(x)=YH^{\ast}(X,X)^{-1}H^{\ast}(X,x)^{T}$. In particular, \citet{Aro(19)} showed that the prediction of the Gaussian NN $\tilde{f}_{m}(W(t),x)$ at $t\rightarrow+\infty$, for $m$ sufficiently large, is equivalent to the kernel regression prediction $f^{\ast}(x)$. Within the $\alpha$-Stable setting, it is not clear whether the fluctuations of  $\tilde{H}_{m}(W(t),X)=\tilde H_m^{(1)}(W(t),X)+\tilde H_m^{(2)}(W(t),X)$ during the training  vanish, as $m\rightarrow\infty$. Theorem \ref{teo_ntk0} shows that the fluctuations of $\tilde{H}_m^{(2)}(W(t),X)$ vanish, as $m\rightarrow\infty$. This result is based on the fact that for every $\delta>0$ it holds that
	$$
	(\log m)^{2/\alpha}\left\Vert \frac{\partial \tilde f_m}{\partial w}(W,x_j;\alpha)-\frac{\partial \tilde f_m}{\partial w}(W(0),x_j;\alpha)\right\Vert_F^2<cm^{-2\gamma/\alpha},
	$$
	for every $j=1,\dots,k$, and for every $W$ such that $\Vert W-W(0)\Vert_F\leq (\log m)^{2/\alpha}$, with probability at least $1-\delta$, if $m$ is sufficiently large. See Lemma \ref{lem_stab3}. The same property is not true if the partial derivatives with respect to $w$ are replaced by the partial derivatives with respect to $w^{(0)}$. Therefore, it is not clear whether the fluctuations of $\tilde H_m^{(1)}(W(t),X)$ during training also vanish, as $m\rightarrow\infty$. 
	
Another interesting avenue for future research would be to extend our results to deep $\alpha$-Stable NNs, for a general depth $D\geq2$. In particular, consider the following setting: i) for $d,k\geq1$ let $X$ be the $d\times k$ NN's input, with $x_j=(x_{j1},\ldots,x_{jd})^{T}$ being the $j$-th input (column vector); ii) for $D,m\geq1$ and $n\geq1$ let: i) $(W^{(1)},\ldots,W^{(D)})$ be the NN's weights such that $W^{(1)}=(w^{(1)}_{1,1},\ldots,w^{(1)}_{m,d})$ and $W^{(l)}=(w^{(l)}_{1,1},\ldots,w^{(1)}_{m,m})$ for $2\leq l\leq D$, where the $w^{(l)}_{i,j}$'s are i.i.d. as an $\alpha$-Stable distribution with scale $\sigma>0$, e.g. assume $\sigma=1$. Then,
\begin{displaymath}
f_{i}^{(1)}(X;\alpha)=\sum_{j=1}^{d}w_{i,j}^{(1)}x_{j}
\end{displaymath}
and
\begin{displaymath}
f_{i,m}^{(l)}(X;\alpha)=\sum_{j=1}^{m}w_{i,j}^{(l)}f_{j}^{(l-1)}(X,m)I(f_{j}^{(l-1)}(X,m)>0)
\end{displaymath}
with $f_{i,m}^{(1)}(X;\alpha):=f_{i}^{(1)}(X;\alpha)$, is a deep $\alpha$-Stable ReLU-NN of depth $D$ and width $m$. Under the assumption that the NN's width grows sequentially over the NN's layers, i.e. $m\rightarrow+\infty$ one layer at a time, it is easy to extend Theorem \ref{teo_priorlimit} to $f_{i,m}^{(l)}(X;\alpha)$. Under the same assumption on the growth of $m$, we expect the NTK analysis of deep $\alpha$-Stable ReLU-NNs to follow along lines similar to that we have developed for shallow $\alpha$-Stable ReLU-NN, though computations may be more involved. A more challenging task would to extend our results to deep $\alpha$-Stable ReLU-NNs under the assumptions that the NN's width grows jointly over the NN's layers, i.e. $m\rightarrow+\infty$ simultaneously over the layers.


\appendix


\section{}\label{app_limit}

Throughout this section, it is assumed that all the random variables are defined on a common probability space, say $(\Omega,\mathcal{F},\P)$, unless otherwise stated.

We make use several times of the following characterization of the spectral measure of $\alpha$-stable distributions: if $S\sim\text{St}_{k}(\alpha,\Gamma)$, then for every Borel set $B$ of $\mathbb S^{k-1}$ such that $\Gamma(\partial B)=0$, it holds true that
$$
\lim_{r\rightarrow\infty}r^\alpha\P\left(\Vert S\Vert>r,\frac{S}{\Vert S\Vert}\in B \right)=C_\alpha\Gamma(B),
$$
where
$$
C_\alpha=\left\{\begin{array}{ll}
	\frac{1-\alpha}{\Gamma(2-\alpha)\cos (\pi\alpha/2)}&\alpha\neq 1\\[0.2cm]
	\frac{2}{\pi}&\alpha=1.
\end{array}
\right.
$$
The proof is reported in Appendix \ref{appendixid} for completeness.
Moreover,  the distribution of a random vector $\xi$ belongs to the domain of attraction of the $\text{St}_{k}(\alpha,\Gamma)$ distribution, with $\alpha\in (0,2)$ and $\Gamma$ simmetric finite measure on $\mathbb S^{k-1}$, if and only if
\begin{equation}
	\label{eq:levy}
	\lim_{n\rightarrow \infty}n\mathbb{P}\left(||\xi||>n^{1/\alpha} ,\frac{\xi}{||\xi||}\in A\right)=C_\alpha \Gamma(A)
\end{equation}
for every  Borel set $A$ of $S$ such that $\Gamma (\partial A)=0$. (See Appendix \ref{appendixid} for more details).

\subsection{Proof of Theorem \ref{teo_priorlimit}}\label{proofTh1}
To simplify the notation, we set in this section: $w:=w(0)$, $w^{(0)}:=w^{(0)}(0)$, and $W:=W(0)$. First, we will prove that $[\langle w_i^{(0)},x_j\rangle I(\langle w_i^{(0)},x_j\rangle >0)]_j$
belongs to the domain of attraction of an $\alpha$-stable law with spectral measure 
\begin{align*}
	\Gamma_1&= C_\alpha \E_{u\sim \Gamma_0}\left(\Vert[\langle u,x_j\rangle I(\langle u,x_j\rangle >0)]_j
	\Vert^\alpha \delta\biggl(
	\frac{
		[\langle u,x_j\rangle I(\langle u,x_j\rangle >0)]_j
	}{
		\Vert[\langle u,x_j\rangle I(\langle u,x_j\rangle >0)]_j
		\Vert}\biggr) 
	\right),
	\end{align*}
where  $ \Gamma_0$ is the spectral measure of $w_i^{(0)}$.
For this, it is sufficient to show that
\begin{align*}
	&r^\alpha \P\left(
	\frac{
		[\langle w_i^{(0)},x_j\rangle I(\langle w_i^{(0)},x_j\rangle >0)]_j
	}{
		\Vert[\langle w_i^{(0)},x_j\rangle I(\langle w_i^{(0)},x_j\rangle >0)]_j
		\Vert}\in B , \Vert[\langle w_i^{(0)},x_j\rangle I(\langle w_i^{(0)},x_j\rangle >0)]_j
	\Vert>r
	\right)\\
	&\quad\rightarrow C_\alpha \Gamma_1(B),
\end{align*}
for every Borel set $B$ of $\mathbb S^{k-1}$ such that $\Gamma_1(\partial B)=0$ (see Appendix \ref{appendixid}).
Let  $T:\mathbb{S}^{k-1}\mapsto \mathbb [0,1]^k$ and $C:\mathbb R^k\setminus \{0\}\rightarrow \mathbb S^{k-1}$ be defined as $T(u)=[\langle u,x_j\rangle I(\langle u,x_j\rangle>0]_j$
and $C(v)=v/\Vert v\Vert$, respectively. Fix a Borel set $B$ of $\mathbb S^{k-1}$ such that $\Gamma_1(\partial B)=0$. This condition implies that
\begin{align*}
&\Gamma_0\left(
\left\{u\in \mathbb S^{k-1}:\Vert T(u)\Vert \neq 0, T(u)\in C^{-1}(\partial B)\right\}
\right)\\
&\quad \quad =
\Gamma_0\left(
\left\{u\in \mathbb S^{k-1}:\Vert T(u)\Vert \neq 0, \frac{T(u)}{\Vert T(u)\Vert}\in\partial B\right\}
\right)=0.
\end{align*}
Hence
\begin{align*}
&\Gamma_0\left(
T^{-1}\left(\left\{z\in [0,1]^k:\Vert z\Vert \neq 0,z\in \partial C^{-1}(B)\right\}
\right)\right)\\
&\quad 
=\Gamma_0\left(
T^{-1}\left(\left\{z\in [0,1]^k:\Vert z\Vert\neq 0,z\in  C^{-1}(\partial B)\right\}
\right)\right)=0.
\end{align*}
Now, let $Z=T(w_i^{(0)}/\Vert w_i^{(0)}\Vert)I(\Vert w_i^{(0)}\Vert \neq 0)$.
We can write that
\begin{align*}
	&r^\alpha \P\left(
	\frac{
		[\langle w_i^{(0)},x_j\rangle I(\langle w_i^{(0)},x_j\rangle >0)]_j
	}{
		\Vert[\langle w_i^{(0)},x_j\rangle I(\langle w_i^{(0)},x_j\rangle >0)]_j
		\Vert}\in B, \Vert[\langle w_i^{(0)},x_j\rangle I(\langle w_i^{(0)},x_j\rangle >0)]_j
	\Vert>r
	\right)\\
	&\quad =r^\alpha \P\biggl(\Vert Z\Vert\neq 0,
	\frac{
		Z
	}{\Vert Z\Vert
	}\in B, \Vert w_i^{(0)}\Vert \Vert Z\Vert>r
	\biggr)\\
	&\quad=\int_{C^{-1}(B)\cap [0,1]^k} r^\alpha \P(
	\Vert w_i^{(0)}\Vert>r\Vert z\Vert^{-1}, Z\in dz)\\
	&\quad=\int_{C^{-1}(B)\cap [0,1]^k}\Vert z\Vert^\alpha (r\Vert z\Vert^{-1})^\alpha \P(
	\Vert w_i^{(0)}\Vert>r\Vert z\Vert^{-1}, \frac{w_i^{(0)}}{\Vert w_i^{(0)}\Vert }\in T^{-1}(dz)).
\end{align*}
Since $\Gamma_0\left(
T^{-1}\left(\left\{z\in [0,1]^k:z\neq 0,z\in \partial (C^{-1}(B))\right\}
\right)\right)=0$, then the points of discontinuity of the function
$\Vert z\Vert^\alpha I(C^{-1}(B))(z)$ have zero $\Gamma_0(T^{-1}(\cdot))$-measure. It follows that 
\begin{align*}
	&\int_{C^{-1}(B)\cap [0,1]^k}\Vert z\Vert^\alpha (r\Vert z\Vert^{-1})^\alpha \P(
	\Vert w_i^{(0)}\Vert>r\Vert z\Vert^{-1}, w_i^{(0)}\in T^{-1}(dz))\\
	&\quad \rightarrow C_\alpha \int_{C^{-1}(B)\cap [0,1]^k} \Vert z\Vert^\alpha \Gamma_0(T^{-1}(dz))\\
	&\quad = C_\alpha \int_{\mathbb S^{k-1}} I(u\in B)\left(\frac{T(u)}{\Vert T(u)\Vert}\right)\Vert T(u)\Vert^\alpha \Gamma_0(du)\\
	&\quad =C_\alpha \Gamma_1(B),
\end{align*}
as $r\rightarrow\infty$,
which completes the proof that
$[\langle w_i^{(0)},x_j\rangle I(\langle w_i^{(0)},x_j\rangle >0)]_j$
belongs to the domain of attraction of an $\alpha$-stable law with spectral measure $\Gamma_1$.
Then, for every $k$-dimensional vector $s$,
$$
\frac{1}{m^{1/\alpha}}\sum_{i=1}^m \sum_{j=1}^ks_j \langle w_i^{(0)},x_j\rangle I(\langle w_i^{(0)},x_j\rangle >0),
$$
as a sequence of random variables in $m$, converges in distribution, as $m\rightarrow+\infty$, to a random variable with $\alpha$-stable distribution and characteristic function
$$
\exp\biggl(-|t|^\alpha
\E_{u\sim \Gamma_0}\bigl(|\sum_{j=1}^ks_j \langle u,x_j\rangle I(\langle u,x_j\rangle >0)|^\alpha\bigr)
\biggr).
$$
Thus, the distribution of
$
\sum_{j=1}^ks_j \langle w_i^{(0)},x_j\rangle I(\langle w_i^{(0)},x_j\rangle >0)
$
belongs to the domain of attraction of an $\alpha$-stable law. In particular, this implies that as $m\rightarrow+\infty$
\begin{align*}
&r^\alpha \P\biggl(|\sum_{j=1}^ks_j \langle w_i^{(0)},x_j\rangle I(\langle w_i^{(0)},x_j\rangle >0)|>r\biggr)\\
&\quad\rightarrow C_\alpha \E_{u\sim \Gamma_0}\biggl(|\sum_{j=1}^ks_j \langle u,x_j\rangle I(\langle u,x_j\rangle >0)|^\alpha\biggr).
\end{align*}
By \citet[Theorem 4]{Cli(86)} with $\beta=\gamma=0$,
\begin{align*}
\P\biggl(
&|w_i|\;\;|\sum_{j=1}^ks_j \langle w_i^{(0)},x_j\rangle I(\langle w_i^{(0)},x_j\rangle >0)|>e^t\biggr)\\
&\quad\sim C_\alpha^2 \E_{u\sim \Gamma_0}\bigl(|\sum_{j=1}^ks_j \langle u,x_j\rangle I(\langle u,x_j\rangle >0)|^\alpha\bigr)\alpha te^{-\alpha t}
\end{align*}
as $t\rightarrow\infty$.
Thus, for $r\rightarrow\infty$,
\begin{align*}
&r^\alpha \P\biggl(|w_i|\;\;|\sum_{j=1}^ks_j \langle w_i^{(0)},x_j\rangle I(\langle w_i^{(0)},x_j\rangle >0)|>r \biggr)\\
&\quad\sim C_\alpha^2 \E_{u\sim \Gamma_0}\bigl(|\sum_{j=1}^ks_j \langle u,x_j\rangle I(\langle u,x_j\rangle >0)|^\alpha\bigr)\alpha \log r.
\end{align*}
Let $\tilde L(r)=C_\alpha^2 \E_{u\sim\Gamma_0}\bigl(|\sum_{j=1}^ks_j \langle u,x_j\rangle I(\langle u,x_j\rangle >0)|^\alpha\bigr)\alpha \log r.$ Since the distribution of
$
w_i\sum_{j=1}^ks_j \langle w_i^{(0)},x_j\rangle I(\langle w_i^{(0)},x_j\rangle >0)
$ is symmetric, then we can write that 
$$
\frac{1}{a_m}\sum_{i=1}^m
w_i\sum_{j=1}^ks_j \langle w_i^{(0)},x_j\rangle I(\langle w_i^{(0)},x_j\rangle >0),
$$
as a sequence of random variables in $m$, converges in distribution, as $m\rightarrow+\infty$, to a random variable with symmetric $\alpha$-stable law with scale $1$ provided $(a_m)_{m\geq1}$ satisfies
$$
\frac{m\tilde L(a_m)}{a_m^\alpha}\rightarrow C_\alpha
$$
as $m\rightarrow\infty$. The condition is satisfied if 
$$
a_m=\left(C_\alpha \E_{u\sim \Gamma_0}\bigl(|\sum_{j=1}^ks_j \langle u,x_j\rangle I(\langle u,x_j\rangle >0)|^\alpha\bigr)
m\log m\right)^{1/\alpha}.
$$
It follows that
$$
\frac{1}{(m\log m)^{1/\alpha}}
\sum_{i=1}^m w_i \sum_{j=1}^ks_j \langle w_i^{(0)},x_j\rangle I(\langle w_i^{(0)},x_j\rangle >0),
$$
as a sequence of random variables in $m$, converges in distribution, as $m\rightarrow+\infty$, to a random variable with symmetric $\alpha$-stable distribution with scale of the form
$$
 \left(C_\alpha \E_{u\sim \Gamma_0}\bigl(|\sum_{j=1}^ks_j \langle u,x_j\rangle I(\langle u,x_j\rangle >0)|^\alpha\bigr)\right)^{1/\alpha}.
$$
Since this holds for every vector $s$, then
$$
\frac{1}{(m\log m)^{1/\alpha}}
\sum_{i=1}^m w_i[\langle w_i^{(0)},x_j\rangle I(\langle w_i^{(0)},x_j\rangle >0)]_j,
$$
as a sequence of random variables in $m$, converges in distribution, as $m\rightarrow+\infty$, to a random vector with symmetric $\alpha$-stable law with the spectral measure
\begin{align*}
\Gamma_{X}=\frac 1 2 C_\alpha &\E_{u\sim\Gamma_0}\Biggl(\Vert[\langle u,x_j\rangle I(\langle u,x_j\rangle >0)]_j
\Vert^\alpha \\
&\delta\biggl(
\frac{
[\langle u,x_j\rangle I(\langle u,x_j\rangle >0)]_j
}{
\Vert[\langle u,x_j\rangle I(\langle u,x_j\rangle >0)]_j
\Vert}\biggr)
+\delta\biggl(-
\frac{
[\langle u,x_j\rangle I(\langle u,x_j\rangle >0)]_j
}{
\Vert[\langle u,x_j\rangle I(\langle u,x_j\rangle >0)]_j
\Vert}\biggr)\Biggr).
\end{align*}
Since $\Gamma_0=\frac 1 2 \sum_{i=1}^d( \delta(e_i)+\delta(-e_i))$, where $e_{ij}=1$ if $j=i$ and $0$ otherwise, then 
$$
\Gamma_{X}=\frac{C_\alpha}{4}\sum_{i=1}^d \left(
\Vert
[x_{ji}I(x_{ji}>0)]_j\Vert^\alpha
\left(
\delta\bigl(\dfrac{[x_{ji}I(x_{ji}>0)]_j}{\Vert
[x_{ji}I(x_{ji}>0)]_j\Vert}\bigr)
+\delta\bigl(-\dfrac{[x_{ji}I(x_{ji}>0)]_j}{\Vert
[x_{ji}I(x_{ji}>0)]_j\Vert}\bigr)
\right)\right.
$$
$$\left.
\quad\quad\quad+
\Vert
[x_{ji}I(x_{ji}<0)]_j\Vert^\alpha
\left(
\delta\bigl(\dfrac{[x_{ji}I(x_{ji}<0)]_j}{\Vert
[x_{ji}I(x_{ji}<0)]_j\Vert}\bigr)
+\delta\bigl(-\dfrac{[x_{ji}I(x_{ji}<0)]_j}{\Vert
[x_{ji}I(x_{ji}<0)]_j\Vert}\bigr)
\right)
\right).
$$




\subsection{Proof of Theorem \ref{grad_asy}}\label{proofTh2}
To simplify the notation, we set in this section: $w:=w(0)$, $w^{(0)}:=w^{(0)}(0)$, $W:=W(0)$, $\tilde{H}^{(1)}_{m}:=\tilde{H}^{(1)}_{m}(W(0),X)$ and $\tilde{H}^{(2)}_{m}:=\tilde{H}^{(2)}_{m}(W(0),X)$, with $\tilde{H}^{(1)}_{m}(W,X)$ and $\tilde{H}^{(2)}_{m}(W,X)$ defined in \eqref{kernel_1} and \eqref{kernel_2}. The proof of Theorem \ref{grad_asy} is split into several steps.
 \begin{lem}\label{grad_asy_1}
 If $m\rightarrow+\infty$ then
	\begin{displaymath}
		\tilde{H}^{(1)}_{m}\stackrel{\text{w}}{\longrightarrow}\tilde{H}^{\ast}_{1}(\alpha),
	\end{displaymath}
	where $\tilde{H}^{\ast}_{1}(\alpha)$ is an $(\alpha/2)$-Stable positive semi-definite random matrix with spectral measure
	\begin{displaymath}
		\Gamma^{\ast}_{1}=C_{\alpha/2}\sum_{u\in\{0,1\}^k}
		\mathbb P(w_i^{(0)}\in B_u)(\sum_{j,j'}\langle x_j,x_{j'}\rangle^2u_ju_{j'})^{\alpha/4} \delta\left(\frac{
			\left[\langle x_j,x_{j'}\rangle u_ju_{j'}\right]_{j,j'}
		}{(\sum_{j,j'}\langle x_j,x_{j'}\rangle^2u_ju_{j'})^{1/2}}
		\right),
	\end{displaymath}
	where, for every $u\in\{0,1\}^k$, $B_u=\{v\in \mathbb R^d:\langle v,x_j\rangle>0\mbox{ if  }u_j=1, \langle v,x_j\rangle\leq 0 \mbox{ if  }u_j=0,j=1,\dots,k \}$, and $C_{\alpha/2}$ is the constant defined in Equation \eqref{const_stabl}.
\end{lem}
\begin{proof}
	The proof follows from a direct application of results in \citet{Cli(86)}. In particular, by  \citet[Lemma 1]{Cli(86)}, as $m\rightarrow+\infty$
\begin{displaymath}
\tilde{H}^{(1)}_{m}\stackrel{\text{w}}{\longrightarrow} \tilde{H}_{1}^{\ast}(\alpha),
\end{displaymath}
where $\tilde{H}^{\ast}_{1}(\alpha)$ is an $(\alpha/2)$-Stable random matrix with spectral measure $\Gamma_{1}^{\ast}$ of the form
\begin{align*}
&\Gamma^{\ast}_{1}=C_{\alpha/2}
\E\Biggl(
\Vert
[\langle x_j,x_{j'}\rangle I(\langle w^{(0)}_i,x_{j'}\rangle >0)
]_{j,j'}
\Vert_F^{\alpha/2}
\delta\biggl(
\frac{
	[\langle x_j,x_{j'}\rangle I(\langle w^{(0)}_i,x_{j'}\rangle >0)
]_{j,j'}}{
\Vert
	[\langle x_j,x_{j'}\rangle I(\langle w^{(0)}_i,x_{j'}\rangle >0)
]_{j,j'}
\Vert_F
}
\biggr)
\Biggr)\\
&=C_{\alpha/2}\sum_{u\in\{0,1\}^k}
\mathbb P(w_i^{(0)}\in B_u)(\sum_{j,j'}\langle x_j,x_{j'}\rangle^2u_ju_{j'})^{\alpha/4} \delta\left(\frac{
	\left[\langle x_j,x_{j'}\rangle u_ju_{j'}\right]_{j,j'}
}{(\sum_{j,j'}\langle x_j,x_{j'}\rangle^2u_ju_{j'})^{1/2}}
\right).
\end{align*}
We will now prove that $\tilde{H}^{\ast}_{1}(\alpha)$ is positive semi-definite. By definition,  $\tilde H_m^{(1)}(\omega)$ is positive semi-definite for every $\omega$ and every $m$. 
By Portmanteau Theorem, for every vector $u\in\mathbb S^{k-1},$
$$
\P\left(
u^T \tilde{H}^{\ast}_{1}(\alpha) u\geq 0\right)\geq \limsup_m \P\left(u^T \tilde H_m^{(1)}\;u\geq 0 \right)= 1.
$$
Let $\cal A$ be a countable dense subset of $\mathbb S^{k-1}$. Then, with probability one, $a^T \tilde{H}^{\ast}_{1}(\alpha)a\geq 0$ for every $a\in\mathcal A$. By continuity, this implies that the same property holds true with probability one for every $u\in\mathbb S^{k-1}$, which proves that $\tilde{H}^{\ast}_{1}(\alpha)$ is almost surely positive semi-definite. By eventually modifying $\tilde{H}^{\ast}_{1}(\alpha)$ on a null set, we obtain a positive semi-definite random matrix.
\end{proof}

\begin{lem}\label{grad_asy_2}
	 If $m\rightarrow+\infty$ then
	\begin{displaymath}
		\tilde{H}^{(2)}_{m}\stackrel{\text{w}}{\longrightarrow}\tilde{H}^{\ast}_{2}(\alpha),
	\end{displaymath}
	where $\tilde{H}^{\ast}_{2}(\alpha)$ is an $(\alpha/2)$-Stable positive semi-definite random matrix with spectral measure
	\begin{displaymath}
			\Gamma^*_2=C_{\alpha/2}\sum_{u\in \{0,1\}^k}\sum_{\{i:\{e_i,-e_i\}\cap B_u\neq\emptyset\}}(\sum_j x_{ji}^2u_j)^{\alpha/2}\delta\left(
		\frac{[x_{ji}u_jx_{j'i}u_{j'}]_{j,j'}}{\sum_jx_{ji}^2u_j}
		\right),
	\end{displaymath}
	where  $B_u=\{v\in \mathbb R^d:\langle v,x_j\rangle>0\mbox{ if  }u_j=1, \langle v,x_j\rangle\leq 0 \mbox{ if  }u_j=0,j=1,\dots,k \}$,  $e_i$ is a $d$-dimensional vector satisfying $e_{ij}=1$ if $j=i$, and $e_{ij}=0$ if $j\neq i$  $(i,j=1,\dots,d)$, and $C_{\alpha/2}$ is the constant defined in Equation \eqref{const_stabl}.
	\end{lem}
\begin{proof}
By the properties of the multivariate stable distribution (see Appendix \ref{appendixid}), it is sufficient to show that
\begin{align*}
	&\P\left(
	\frac{\left[
		\langle w_1^{(0)},x_j\rangle \langle w_1^{(0)},x_{j'}\rangle 
		I(\langle w_1^{(0)},x_j\rangle>0)I(\langle w_1^{(0)},x_{j'}\rangle>0)
		\right]_{j,j'}}
	{\Vert\left[
		\langle w_1^{(0)},x_j\rangle \langle w_1^{(0)},x_{j'}\rangle 
		I(\langle w_1^{(0)},x_j\rangle>0)I(\langle w_1^{(0)},x_{j'}\rangle>0)
		\right]_{j,j'}\Vert_F}\in\cdot, \right.\\
	&\quad\quad\quad \left.\Vert\left[
	\langle w_1^{(0)},x_j\rangle \langle w_1^{(0)},x_{j'}\rangle 
	I(\langle w_1^{(0)},x_j\rangle>0)I(\langle w_1^{(0)},x_{j'}\rangle>0)
	\right]_{j,j'}\Vert_F>r
	\right)\\
	&\quad\quad\quad\sim C_{\alpha/2}r^{-\alpha/2}\Gamma_2^*(\cdot),
\end{align*}
	as $r\rightarrow+\infty$.
	We can write that
\begin{align*}
	&\P\left(
	\frac{\left[
		\langle w_1^{(0)},x_j\rangle \langle w_1^{(0)},x_{j'}\rangle 
		I(\langle w_1^{(0)},x_j\rangle>0)I(\langle w_1^{(0)},x_{j'}\rangle>0)
		\right]_{j,j'}}
	{\Vert\left[
		\langle w_1^{(0)},x_j\rangle \langle w_1^{(0)},x_{j'}\rangle 
		I(\langle w_1^{(0)},x_j\rangle>0)I(\langle w_1^{(0)},x_{j'}\rangle>0)
\right]_{j,j'}\Vert_F}\in\cdot, \right.\\
&\quad\quad\quad \left.\Vert\left[
\langle w_1^{(0)},x_j\rangle \langle w_1^{(0)},x_{j'}\rangle 
I(\langle w_1^{(0)},x_j\rangle>0)I(\langle w_1^{(0)},x_{j'}\rangle>0)
\right]_{j,j'}\Vert_F>r
	\right)\\
	&=\sum_{u\in\{0,1\}^k}
	\P\left(
	\frac{\left[
		\langle w_1^{(0)},u_jx_j\rangle \langle w_1^{(0)},u_{j'}x_{j'}\rangle 
				\right]_{j,j'}}
	{\Vert\left[
		\langle w_1^{(0)},u_jx_j\rangle \langle w_1^{(0)},u_{j'}x_{j'}\rangle 
						\right]_{j,j'}\Vert_F}\in\cdot, \right.\\
	&\quad\quad\quad \left.\Vert\left[
	\langle w_1^{(0)},u_jx_j\rangle \langle w_1^{(0)},u_{j'}x_{j'}\rangle
		\right]_{j,j'}\Vert_F>r, w_1^{(0)}\in B_u
	\right).
\end{align*}
For every $u\in\{0,1\}^k$, let $X_u$ be the $d\times k$ matrix, defined as 
$$
X_u=[x_{ji}u_j]_{j=1,\dots,k,i=1,\dots,d}.
$$ 
Then we can write that
\begin{align*}
	&\P\left(
	\frac{\left[
		\langle w_1^{(0)},u_jx_j\rangle \langle w_1^{(0)},u_{j'}x_{j'}\rangle 
		\right]_{j,j'}}
	{\Vert\left[
		\langle w_1^{(0)},u_jx_j\rangle \langle w_1^{(0)},u_{j'}x_{j'}\rangle 
		\right]_{j,j'}\Vert_F}\in\cdot, \right.\\
	&\quad\quad\quad \left.\Vert\left[
	\langle w_1^{(0)},u_jx_j\rangle \langle w_1^{(0)},u_{j'}x_{j'}\rangle
	\right]_{j,j'}\Vert_F>r, w_1^{(0)}\in B_u
	\right)\\
	&=\P\left(
	\frac{X_u^Tw_1^{(0)}(w_1^{(0)})^TX_u}
	{(\mathrm{tr}(X_u^T(w_1^{(0)})^Tw_1^{(0)}X_uX_u^T
	(w_1^{(0)})^T	w_1^{(0)}X_u))^{1/2}}\in\cdot, \right.\\
	&\quad\quad\quad\quad\quad\quad \left.\mathrm{tr}(X_u^T(w_1^{(0)})^Tw_1^{(0)}X_uX_u^T(w_1^{(0)})^Tw_1^{(0)}X_u)>r^2, w_1^{(0)}\in B_u
	\right)\\
	&=\P\left(
	\frac{X_u^T(w_1^{(0)})^Tw_1^{(0)}X_u}
	{w_1^{(0)}X_uX_u^T(w_1^{(0)})^T}\in\cdot, w_1^{(0)}X_uX_u^T(w_1^{(0)})^T>r, w_1^{(0)}\in B_u
	\right).
	\end{align*}
Notice that the maximum eigenvalue of the matrix $X_uX_u^T$ is smaller than or equal to $k$, since the norm of each column of $X_u$ is smaller than or equal to one. Then $w_1^{(0)}X_uX_u^T(w_1^{(0)})^T>r$ implies that $\Vert w_1^{(0)}\Vert >(r/k)^{1/2}$. We can therefore write that
\begin{align*}
	&\P\left(
	\frac{X_u^T(w_1^{(0)})^Tw_1^{(0)}X_u}
	{w_1^{(0)}X_uX_u^T(w_1^{(0)})^T}\in\cdot, w_1^{(0)}X_uX_u^T(w_1^{(0)})^T>r, w_1^{(0)}\in B_u
	\right)\\
	&=\P\left(
	\frac{X_u^T(w_1^{(0)})^Tw_1^{(0)}X_u}
	{w_1^{(0)}X_uX_u^T(w_1^{(0)})^T}\in\cdot, w_1^{(0)}X_uX_u^T(w_1^{(0)})^T>r, \Vert w_1^{(0)}\Vert>(r/k)^{1/2}, w_1^{(0)}\in B_u
	\right).
\end{align*}
Since $B_u$ is a cone and the spectral measure of $w_1^{(0)}$ is given by $\sum_i(\delta(e_i)+\delta(-e_i))$, by the properties of the multivariate stable distribution, we can write that
\begin{align*}
	&\P\left(
	\frac{X_u^T(w_1^{(0)})^Tw_1^{(0)}X_u}
	{w_1^{(0)}X_uX_u^T(w_1^{(0)})^T}\in\cdot, w_1^{(0)}X_uX_u^T(w_1^{(0)})^T>r, \Vert w_1^{(0)}\Vert>(r/k)^{1/2}, w_1^{(0)}\in B_u
	\right)\\
	&\quad\quad \sim C_{\alpha/2}r^{-\alpha/2}\sum_{\{i:\{e_1,-e_i\}\cap B_u\neq\emptyset\}}(\sum_{j=1}^k x_{ji}^2u_j)^{\alpha/2}\delta\left(
	\frac{[x_{ji}x_{j'i}u_ju_{j'}]_{j,j'}}{\sum_jx_{ji}^2u_j}
	\right),
\end{align*}
as $r\rightarrow+\infty$. The proof that $\tilde H_2^\ast(\alpha)$ is positive semi-definite can be done by following the same line of reasoning as in the proof of Lemma \ref{grad_asy_1}.
\end{proof}


\begin{lem}
	\label{lem:joint}
		As $m\rightarrow+\infty$, the probability distribution of 
	$(\tilde{H}^{(1)}_{m},\tilde{H}^{(1)}_{m})$ converges weakly to the  law of independent stable random matrices, with spectral measures $\Gamma^\ast_1$ and $\Gamma^\ast_2$ as in \eqref{spectral_1} and \eqref{spectral2}, respectively.
\end{lem}
\begin{proof}
	Since $\tilde H^{(1)}_m$ and $H^{(2)}_m$ converge marginally to $\alpha/2$-stable random matrices, by the properties of the multivariate stable distributions it is sufficient to show that they converge to stochastically independent random matrices. By Theorem \ref{th:clt}, we know that
\begin{align*}n\P\Biggl(&
	\Vert 
	[w_i^2 \langle x_j,x_{j'}\rangle I(\langle w_i^{(0)},x_j\rangle >0)I(\langle w_i^{(0)},x_{j'}\rangle >0 )]_{j,j'}
	\Vert_F>n^{2/\alpha},\\
	&\Vert 
	[\langle x_j,w^{(0)}_i\rangle\langle x_{j'},w^{(0)}_i\rangle I(\langle w_i^{(0)},x_j\rangle >0)I(\langle w_i^{(0)},x_{j'}\rangle >0 )]_{j,j'}
	\Vert_F>n^{2/\alpha}
	\Biggr)
\end{align*}
and
$$
n\P\Biggl(
	\Vert 
	[w_i^2 \langle x_j,x_{j'}\rangle I(\langle w_i^{(0)},x_j\rangle >0)I(\langle w_i^{(0)},x_{j'}\rangle >0 )]_{j,j'}
	\Vert_F>n^{2/\alpha}
	\Biggr)
	$$
converge to finite limits, as $n\rightarrow\infty$. 	Hence, again by Theorem \ref{th:clt}, it is sufficient to show that
\begin{align*}
&\lim_{n\rightarrow\infty}n\P\Biggl(
\Vert 
[w_i^2 \langle x_j,x_{j'}\rangle I(\langle w_i^{(0)},x_j\rangle >0)I(\langle w_i^{(0)},x_{j'}\rangle >0 )]_{j,j'}
\Vert_F>n^{2/\alpha},\\
&\Vert 
[\langle x_j,w^{(0)}_i\rangle\langle x_{j'},w^{(0)}_i\rangle I(\langle w_i^{(0)},x_j\rangle >0)I(\langle w_i^{(0)},x_{j'}\rangle >0 )]_{j,j'}
\Vert_F>n^{2/\alpha}
\Biggr)=0,
\end{align*}
which ensures that the L\'evy measure of the limit infinitely divisible distribution of $(\tilde H_m^{(1)},\tilde H_m^{(2)})$ is the sum of a measure $\nu_1$  concentrated on the space spanned by the first $k^2$ coordinates and a measure $\nu_2$ on the space spanned by the last $k^2$ coordinates.
We can write that 
\begin{align*}
	&n\P\Biggl(
	\Vert 
	[w_i^2 \langle x_j,x_{j'}\rangle I(\langle w_i^{(0)},x_j\rangle >0)I(\langle w_i^{(0)},x_{j'}\rangle >0 )]_{j,j'}
	\Vert_F>n^{2/\alpha},\\
	&\quad\quad\quad\Vert 
	[\langle x_j,w^{(0)}_i\rangle\langle x_{j'},w^{(0)}_i\rangle I(\langle w_i^{(0)},x_j\rangle >0)I(\langle w_i^{(0)},x_{j'}\rangle >0 )]_{j,j'}
	\Vert_F>n^{2/\alpha}
	\Biggr)\\
	&=n\sum_{u\in\{0,1\}^k}\P(w_i^{(0)}\in B_u)\\
	&\P\biggl(
	\Vert 
	[w_i^2 \langle x_j,x_{j'} \rangle u_ju_{j'}]_{j,j'}
	\Vert_F>n^{2/\alpha},
	\Vert 
	[\langle x_j,w^{(0)}_i\rangle\langle x_{j'},w^{(0)}_i \rangle u_j u_{j'} ]_{j,j'}
	\Vert_F>n^{2/\alpha}	
	\mid w_i^{(0)}\in B_u\biggr)\\
		&=n\sum_{u\in\{0,1\}^k}\P(w_i^{(0)}\in B_u)\P\biggl(
		\Vert 
		[\langle x_j,w^{(0)}_i\rangle\langle x_{j'},w^{(0)}_i \rangle u_j u_{j'} ]_{j,j'}
		\Vert_F>n^{2/\alpha}	
		\mid w_i^{(0)}\in B_u\biggr)
		\\
	&\quad\quad\quad\P\biggl(
	\Vert 
	[w_i^2 \langle x_j,x_{j'} \rangle u_ju_{j'}]_{j,j'}
	\Vert_F>n^{2/\alpha}\biggr)\\
	&=\sum_{u\in\{0,1\}^k}n\P\biggl(
	\Vert 
	[\langle x_j,w^{(0)}_i\rangle\langle x_{j'},w^{(0)}_i \rangle u_j u_{j'} ]_{j,j'}
	\Vert_F>n^{2/\alpha}	
	, w_i^{(0)}\in B_u\biggr)
	\\
	&\quad\quad\quad\P\biggl(
	\Vert 
	[w_i^2 \langle x_j,x_{j'} \rangle u_ju_{j'}]_{j,j'}
	\Vert_F>n^{2/\alpha}\biggr)\rightarrow 0,
\end{align*}
as $n\rightarrow\infty$.
\end{proof}

\begin{proof}[\bf Proof of Theorem \ref{grad_asy}]
	By Lemma \ref{grad_asy_1}, Lemma \ref{grad_asy_1}, Lemma \ref{lem:joint}, and the properties of stable distributions, $\tilde H_m(W(0),X)$ converges in distribution to a positive semi-definite random matrix, with $(\alpha/2)$-stable distribution, and spectral measure $\Gamma_1^\ast+\Gamma_2^\ast$. 
	\end{proof}


\subsection{Proof of Theorem \ref{th:grad_asy2}}\label{proofTh3}
To simplify the notation, we set in this section: $w:=w(0)$, $w^{(0)}:=w^{(0)}(0)$, $W:=W(0)$, $\tilde{H}^{(1)}_{m}:=\tilde{H}^{(1)}_{m}(W(0),X)$ and $\tilde{H}^{(2)}_{m}:=\tilde{H}^{(2)}_{m}(W(0),X)$, with $\tilde{H}^{(1)}_{m}(W,X)$ and $\tilde{H}^{(2)}_{m}(W,X)$ defined in \eqref{kernel_1} and \eqref{kernel_2}.

From \eqref{kernel}, $\tilde{H}_{m}(W(0),X))$ is the sum of two positive semi-definite random matrices, $\tilde H_m^{(1)}$ and $\tilde H_m^{(2)}$. The following results show that for every $\delta>0$, there exist $\lambda_1>0$ and $\lambda_{2}>0$ such that, for $m$ sufficiently large, with probability at least $1-\delta$
$$
\lambda_{\text{min}}(\tilde{H}_{m}^{(i)})>\lambda_{i}.
$$
with the large-width behaviour of $\tilde{H}^{(i)}_{m}$ being characterized in Lemma \ref{grad_asy_1} and Lemma \ref{grad_asy_2}, through an $(\alpha/2)$-Stable limiting random matrix $\tilde{H}^{\ast}_{i}(\alpha)$ with spectral measure $\Gamma_{i}^{\ast}$ of the form \eqref{spectral_1} and \eqref{spectral2}. To prove that the minumum eigenvales of $\tilde{H}^{(1)}_{m}$ and $\tilde{H}^{(2)}_{m}$ are bounded away from zero, we first need to inspect the characteristics of the distributions of $\tilde{H}^{\ast}_{1}(\alpha)$ and of  $\tilde{H}^{\ast}_{2}(\alpha)$.  This is the content of Lemma \ref{eig_1} and of Lemma \ref{eig2}. Then, the results concerning the minumum eigenvalues of $\tilde{H}^{(1)}_{m}$ and $\tilde{H}^{(2)}_{m}$ are given in Lemma \ref{eig_2} and Lemma \ref{eig_3}.

\begin{lem}\label{eig_1}
	Under the assumptions of Theorem \ref{teo_ntk}, the distribution of the random matrix $\tilde{H}^{\ast}_{1}(\alpha)$ is absolutely continuous in the subspace of the symmetric positive semi-definite matrices with zero entries in the positions $(j,j^{\prime})$ such that $\langle x_j,x_{j'}\rangle =0$, with $j,j^{\prime}\in\{1,\ldots,k\}$, with the topology of Frobenius norm.
\end{lem}
\begin{proof}
	 From \citet{Nol(10)}, it is sufficient to show that
$$
\inf_{s\in \mathbb S^{k^2-1}_0}\int |\langle s,u\rangle |^{\alpha/2}\Gamma_1^\ast(\ddr u)\neq 0,
$$
where $\Gamma_1^\ast$ is the spectral measure \eqref{spectral_1}, $\mathbb S^{k^2-1}_0$ is the unit sphere in the space of the $k\times k$ symmetric matrices  such that $s_{j,j'}=0$ if $\langle x_j,x_{j'}\rangle=0$, with the  Frobenius metric. Now, since
\begin{align*}
    &\int |\langle s,u\rangle |^{\alpha/2}\Gamma_1^\ast(\ddr u)\\
    &=C_{\alpha/2}\E\left(|\sum_{j,j'}s_{j,j'}\langle x_j,x_{j'}\rangle 
    I(\langle \wz_i,x_j\rangle >0)I(\langle \wz_i,x_{j'}\rangle >0)|^{\alpha/2}\right)
\end{align*}
is a continuous function of $s$ that takes value in a compact set, then the minimum is attained. Thus it is sufficient to show that for every $s\in \mathbb S^{k^2-1}_0$,
$$
\E\left(|\sum_{j,j'}s_{j,j'}\langle x_j,x_{j'}\rangle 
    I\langle \wz_i,x_j\rangle >0)I(\langle \wz_i,x_{j'}\rangle >0)|^{\alpha/2}\right)\neq 0.
$$
For every $j$ and every  $u_j\in\{0,1\}$, let $A_j^{u_j}$ be the event $(\langle \wz_i,x_j\rangle >0)$ if $u_j=1$ and its complement if $u_j=0$.
Then
\begin{align*}
   &\E\left(|\sum_{j,j'}s_{j,j'}\langle x_j,x_{j'}\rangle 
    I(\langle \wz_i,x_j\rangle >0)I(\langle \wz_i,x_{j'}\rangle >0)|^{\alpha/2}\right)\\
    &\quad=\sum_{u_1,\dots,u_k}\P(A_1^{u_1}\cap\dots\cap A_k^{u_k}) |\sum_{j,j'}u_ju_{j'}s_{j,j'}\langle x_j,x_{j'}\rangle |^{\alpha/2} .
\end{align*}
Since $x_1,\dots,x_k$ are linearly independent, then for every $u_1,\dots, u_k$, $\P(A_1^{u_1}\cap\dots,A_k^{u_k})>0$.  To prove it, assume, without loss of generality, that $u_i=1$ for every $i$. Since $x_1,\dots,x_k$ are linearly independent, then we can complete the matrix
$
X=[x_1\;\dots\,x_k]
$ by adding $k-d$ columns in such a way that the completed matrix $\tilde X$ is non-singular. 
For every  $d$-dimensional vector $v$ such that $v_1>0,\dots,v_k>0$ there exists a vector $u$ such that $u=(\tilde X^T)^{-1}v$. Thus, 
\begin{align*}
&\{u\in\mathbb R^d:\langle u,x_1\rangle>0 ,\dots,\langle u,x_k\rangle>0\}=\{(\tilde X^T)^{-1}v:v_1>0,\dots,v_k>0\}
\end{align*}
is an open non-empty set. Since $\wz_i$ has independent and identically distributed components, with stable distribution, then 
$$
\P\left(\wz_i\in\{(\tilde X)^{-1}v:v_1>0,\dots,v_k>0\}\right)>0.
$$
This concludes the proof that $\P(A_1^{u_1}\cap\dots,A_k^{u_k})>0$ for every $(u_1,\dots, u_k)\in\{0,1\}^k\}$.
It follows that $\int |\langle s,u\rangle |^{\alpha/2}\Gamma_1^\ast(du)$ is zero if and only if, for every $(u_1,\dots,u_k)\in\{0,1\}^k$, it holds
$$
\sum_{j,j'}u_j,u_{j'}\langle x_j,x_{j'}\rangle s_{j,j'}=0.
    $$
The only solution of the above system of  equations in the space of symmetric matrices $s$ such that $s_{j,j'}=0$ if $\langle x_j,x_{j'}\rangle =0$ is $s=0$, which is not consistent with $\Vert s\Vert_F=1$.
\end{proof}

We observe that the space of the symmetric positive semi-definite matrices with zeros in the entries $(j,j^{\prime})$ such that $\langle x_j,x_{j'}\rangle =0$ contains all the matrices with non-zero diagonal element since $\langle x_j,x_j\rangle=1\neq 0$ for every index $j$. 

\begin{lem}\label{eig_2}
	Under the assumptions of Theorem \ref{teo_ntk}, for every $\delta>0$ there exists $\lambda_{1}>0$ such that with probability at least $1-\delta$
	\begin{displaymath}
		\lambda_{\text{min}}(\tilde{H}^{\ast}_{1}(\alpha))>\lambda_1.
	\end{displaymath}
\end{lem}
\begin{proof}
	Since the distribution of $\tilde{H}^{\ast}_{1}(\alpha)$ is absolutely continuous in the space of symmetric positive semi-definite matrices with zero entries in the positions $j,j'$ such that $\langle x_,x_{j'}\rangle=0$, and since this space contains all the symmetric positive semi-definite matrices with non-zero diagonal entries,
then we can write that $\P(\det(\tilde{H}^{\ast}_{1}(\alpha))=0)=0$.  Moreover, since $\tilde{H}^{\ast}_{1}(\alpha)$ is positive semi-definite, then $\P(\lambda_{\text{min}}(\tilde{H}^{\ast}_{1}(\alpha))>0)=1$.  Thus, for every $\delta>0$, the exists $\lambda_1>0$ such that 
$\P(\lambda_{\text{min}}(\tilde{H}^{\ast}_{1}(\alpha))>\lambda_1)>1-\delta$.  
\end{proof}


\begin{lem}\label{eig2}
	Under the assumptions of Theorem \ref{teo_ntk}, the distribution of the random matrix $\tilde{H}^{\ast}_{2}(\alpha)$ is absolutely continuous in the subspace of the symmetric positive semi-definite matrices, with the topology of Frobenius norm.
\end{lem}
\begin{proof}
 From \citet{Nol(10)}, it is sufficient to show that
$$
\inf_{s\in \mathbb S^{k^2-1}}\int |\langle s,u\rangle |^{\alpha/2}\Gamma_2^\ast(\ddr u)\neq 0,
$$
where $\Gamma_2^\ast$ is the spectral measure \eqref{spectral2}, $\mathbb S^{k^2-1}$ is the unit sphere in the space of the $k\times k$ symmetric positive semi-definite matrices, with the  Frobenius norm.
For every $u\in\{0,1\}^k$, let $B_u=\{v\in \mathbb R^d:\langle v,x_j\rangle>0\mbox{ if  }u_j=1, \langle v,x_j\rangle\leq 0 \mbox{ if  }u_j=0 \}$. Moreover, for every $i=1,\dots, k$, let $e_i$ be a $d$-dimensional random vector satisfying $e_{ij}=1$ for $j=i$ and $e_{ij}=0$ for $j\neq i$.
Finally, let $C_{\alpha/2}$ be the constant defined in Equation \eqref{const_stabl}. Then
\begin{align*}
	&\int |\langle s,u\rangle |^{\alpha/2}\Gamma_2^\ast(\ddr u)=C_{\alpha/2}|\sum_{j,j'}
	s_{j,j'}\sum_{u\in \{0,1\}^k}\sum_{\{i:\{e_i,-e_i\}\cap B_u\neq \emptyset \}}x_{ji}u_jx_{j'i}u_{j'}|^{\alpha/2}.
\end{align*}
Since $\sum_{j,j'}
s_{j,j'}\sum_{u\in \mathcal U}\sum_{E}z_{u,i}x_{ji}u_jx_{j'i}u_{j'}$ 
is continuous as a function of $s$ and $s$ takes values in a compact set, then the minimum is attained. Thus it is sufficient to show that for every $s\in \mathbb S^{k^2-1}$,
$$
\sum_{u\in \{0,1\}^k}\sum_{\{i:\{e_i,-e_i\}\cap B_u\neq \emptyset \}}\sum_{j,j'}
s_{j,j'}x_{ji}u_jx_{j'i}u_{j'}\neq 0.
$$
Since $\Vert s\Vert_F=1$, then $s$ is not the null matrix. Hence there exist $c>0$, a vector $a$ with $\Vert a\Vert =1$ and a positive semi-definite, symmetric matrix $s'$ such that 
$$
s=caa^T+s'.
$$
Since $B_u\cap B_{u'}=\emptyset$, when $u\neq u'$, then, for every $i=1,\dots,d$ and $j=1,\dots,k$, there exists one and only one $u\in\{0,1\}^k$ such that $u_j=1$ and $\{e_i,-e_i\}\cap B_u\neq \emptyset$.  Then we can write that 
 \begin{align*}
 &\sum_{u\in \{0,1\}^k}\sum_{\{i:\{e_i,-e_i\}\cap B_u\neq \emptyset \}}\sum_{j,j'}
 s_{j,j'}x_{ji}u_jx_{j'i}u_{j'}\\
 &\geq c
 \sum_{u\in \{0,1\}^k}\sum_{\{i:\{e_i,-e_i\}\cap B_u\neq \emptyset \}}(\sum_{j}
 a_jx_{ji}u_j)^2\\
&=\sum_{i=1}^d \left((\sum_{j=1}^k a_jx_{ji})^2\sum_{\{u:\{e_i,-e_i\}\cap B_u\neq\emptyset\}}u_j\right)\\
&=\sum_{i=1}^d (\sum_{j=1}^k a_jx_{ji})^2,
 \end{align*}
which is strictly positive, since the $x_j$ are linearly independent, and $\Vert a\Vert=1$.
This concludes the proof.
\end{proof}

\begin{lem}\label{eig_3}
	Under the assumptions of Theorem \ref{teo_ntk}, for every $\delta>0$ there exists $\lambda_{2}>0$ such that with probability at least $1-\delta$
	\begin{displaymath}
		\lambda_{\text{min}}(\tilde{H}^{\ast}_{2}(\alpha))>\lambda_2.
	\end{displaymath}
\end{lem}
\begin{proof}
Since the distribution of $\tilde{H}^{\ast}_{2}(\alpha)$ is absolutely continuous in the space of symmetric positive semi-definite matrices then we can write that $\P(\det(\tilde{H}^{\ast}_{2}(\alpha))=0)=0$.  Moreover, since $\tilde{H}^{\ast}_{2}(\alpha)$ is positive semi-definite, then $\P(\lambda_{\text{min}}(\tilde{H}^{\ast}_{2}(\alpha))>0)=1$.  Thus, for every $\delta>0$, the exists $\lambda_2>0$ such that 
$\P(\lambda_{\text{min}}(\tilde{H}^{\ast}_{2}(\alpha))>\lambda_2)>1-\delta$.  
\end{proof}

\bigskip

\begin{proof}[\bf Proof of Theorem \ref{th:grad_asy2}]
 Let $\delta>0$ be a fixed number. By Lemmas \ref{eig_2} and \ref{eig_3}, there exist $\lambda_1>0$ and $\lambda_2>0$ such that, for $i=1,2$, $\P(\lambda_{\text{min}}(\tilde{H}_i^\ast(\alpha))>\lambda_i)\geq 1-\delta/2$. Since the minimum eigenvalue map is continuous with respect to Frobenius norm then, by Portmanteau theorem,
for $i=1,2$,
$$
\liminf_m\P(\lambda_{\text{min}}(\tilde{H}^{(i)}_{m}(W(0),X))>\lambda_i)\geq
\P(\lambda_{\text{min}}(\tilde{H}_i^\ast(\alpha))>\lambda_i)\geq 1-\delta/2.
$$
Let $\lambda_0=\lambda_1+\lambda_2$. Since the minimum eigenvalue of a sum of symmetric, positive semi-definite matrices is greater than or equal to the sum of the eigenvalues of the two matrices (see \cite{Horn} Theorem 4.3.1), then we can write that 
\begin{align*}
&\liminf_m\P(\lambda_{\text{min}}(\tilde{H}_{m}(W(0),X))>\lambda_0)\\
&\quad\geq
\liminf_m\P(\lambda_{\text{min}}(\tilde{H}_{m}^{(1)}(W(0),X))+\lambda_{\text{min}}(\tilde{H}_{m}^{(2)}(W(0),X))>\lambda_0)\\
&\quad\geq
\liminf_m\P(\cap_{i=1,2}(\lambda_{\text{min}}(\tilde{H}_{m}^{(i)}(W(0),X))>\lambda_i))\\
&\quad\geq 1-\limsup_m\left(
\sum_{i=1}^2\P(\lambda_{\text{min}}(\tilde{H}_{m}^{(i)}(W(0),X))\leq \lambda_i)\right)\\
&\quad\geq 1-\delta,
\end{align*}
thus completing the proof.
\end{proof}

\subsection{Proof of Theorem \ref{teo_ntk0}}\label{proofTh4}
Before proving Theorem \ref{teo_ntk0}, we give some preliminary results.
\begin{lem}\label{lem_stab3}
	Let $\gamma\in (0,1)$ and $c>0$ be fixed numbers. For every $\delta>0$ the following property holds true, for $m$ sufficiently large, with probability at least $1-\delta$:
	$$
	(\log m)^{2/\alpha}\left\Vert \frac{\partial \tilde f_{m}}{\partial w}(W,x_j;\alpha)-\frac{\partial \tilde f_{m}}{\partial w}(W(0),x_j;\alpha)\right\Vert_F^2 < c m^{-2\gamma/\alpha},
	$$
	for every $W$ such that $||W-W(0)||_F\leq (\log m)^{2/\alpha}$ and every NN's input $x_{j}$, with $j=1,\ldots,k$.
\end{lem}
\begin{proof}
For a fixed $W(0)$, let $W$ be such that $\Vert W-W(0)\Vert_F\leq (\log m)^{2/\alpha}$. Then it holds $\Vert \wz-\wz(0)\Vert_F^2\leq \Vert W-W(0)\Vert_F^2\leq (\log m)^{4/\alpha}$. Accordingly, we can write the following
\begin{align*}
  &(\log m)^{2/\alpha}\left\Vert \frac{\partial \tilde f_{m}}{\partial w}(W,x_j;\alpha)-\frac{\partial \tilde f_{m}}{\partial w}(W(0),x_j;\alpha)\right\Vert_F^2\\
  &\quad\leq\frac{1}{m^{2/\alpha}}\sum_{i=1}^m
  \left(\langle \wz_i,x_j\rangle I(\langle \wz_i,x_j\rangle >0)
  -\langle \wz_i(0),x_j\rangle I(\langle \wz_i(0),x_j\rangle >0)\right)^2\\
  &\quad\leq \frac{2}{m^{2/\alpha}}\sum_{i=1}^m
  \left(\langle \wz_i,x_j\rangle -\langle \wz_i(0),x_j\rangle\right)^2 I(\langle \wz_i,x_j\rangle >0)\\
  &\quad\quad+
  \frac{2}{m^{2/\alpha}}\sum_{i=1}^m
  \langle \wz_i(0),x_j\rangle^2 \left(I(\langle \wz_i,x_j\rangle >0)
  -I(\langle \wz_i(0),x_j\rangle >0)\right)^2.
  \end{align*}
We will bound the two terms of the sum separately. 
First, we define $r_i=|\langle \wz_i-\wz_i(0),x_j\rangle|$ for $i=1,\ldots,m$. Then, we can write that
$$
\sum_{i=1}^mr_i^2\leq \sum_{i=1}^m \Vert \wz_i-\wz_i(0)\Vert^2\cdot \Vert x_j\Vert^2\leq \Vert \wz-\wz(0)\Vert_F^2\leq (\log m)^{4/\alpha}.
$$
Since $\gamma<1$, 
\begin{align*}
   &\frac{2}{m^{2/\alpha}}\sum_{i=1}^m
  \left(\langle \wz_i,x_j\rangle -\langle \wz_i(0),x_j\rangle\right)^2 I(\langle \wz_i,x_j\rangle >0)\\
  &\quad\leq 2 m^{-2/\alpha}(\log m)^{4/\alpha}<\frac{c}{4}m^{-2\gamma/\alpha},
\end{align*}
for $m$ sufficiently large. In order to bound the second term, we observe that the following set
$$
\{\wz(0):\exists\wz s.t.|\langle \wz_i-\wz_i(0),x_j\rangle|=r_i,\;I(\langle \wz,x_j\rangle >0)\neq I (\langle \wz(0),x_j\rangle >0)\}
$$
is included in the set $\{\wz_i(0): |\langle \wz_i(0),x_j\rangle|
\leq r_i \}.$ Therefore, we can write that
\begin{align*}
   & \sup_{\sum_ir_i^2\leq \log m}\sup_{|\wz_i-\wz_i(0)|\leq r_i}\frac{2}{m^{2/\alpha}}\sum_{i=1}^m
  \langle \wz_i(0),x_j\rangle^2 \left(I(\langle \wz_i,x_j\rangle >0)
  -I(\langle \wz_i(0),x_j\rangle >0)\right)^2\\
  &\quad\leq\sup_{\sum_ir_i^2\leq \log m}\sup_{|\wz_i-\wz_i(0)|\leq r_i} \frac{2}{m^{2/\alpha}}\sum_{i=1}^m
  \langle \wz_i(0),x_j\rangle^2 I(\langle \wz_i(0),x_j\rangle <r_i)\\
  &\quad\leq\sup_{\sum_ir_i^2\leq \log m}\sup_{|\wz_i-\wz_i(0)|\leq r_i} \frac{2}{m^{2/\alpha}}\sum_{i=1}^m r_i^2\\
  &\quad\leq \frac{1}{m^{2/\alpha}}(\log m)^{4/\alpha}<\frac{c}{4}m^{-2\gamma/\alpha},
\end{align*}
for $m$ sufficiently large.
\end{proof}

\begin{lem}\label{lem_stab4}
	For every $\delta>0$ there exist $\lambda>0$ such that the following two properties hold true, for $m$ sufficiently large, with a probability at least $1-\delta$:
	\begin{itemize}
		\item[i)]
		$$
		\Vert \tilde{H}_{m}^{(2)}(W,X)-\tilde{H}_{m}^{(2)}(W(0),X)\Vert_F<\lambda m^{-\gamma/\alpha};
		$$
		\item[ii)]
		$$
		\lambda_{\text{min}}(\tilde{H}_{m}(W,X))>\frac{\lambda}{2};
		$$
	\end{itemize}
	for every $W$ such that  $\Vert W-W(0)\Vert_F \leq (\log m)^{2/\alpha}$.
\end{lem}
\begin{proof}
By Lemma \ref{eig_3}, for every $\delta>0$ there exists $\lambda$ such that
\begin{displaymath}
	\lambda_{\text{min}}(\tilde{H}^{\ast}_{2}(\alpha))>\lambda
\end{displaymath}
with probability at least $1-\delta/2$. For every vector $W$, we can write that
\begin{align*}
    &|\tilde{H}_{m}^{(2)}(W,X)[i,j]-\tilde{H}_{m}^{(2)}(W(0),X)[i,j]|\\
    &\quad=(\log m)^{2/\alpha}\left|\left\langle\frac{\partial \tilde f_{m}}{\partial w}(W,x_i;\alpha),
\frac{\partial \tilde f_{m}}{\partial w}(W,x_j;\alpha)\right\rangle -
\left\langle\frac{\partial \tilde f_{m}}{\partial w}(W(0),x_i;\alpha),
\frac{\partial \tilde f_{m}}{\partial w}(W(0),x_j;\alpha)\right\rangle\right|\\
&\quad\leq (\log m)^{2/\alpha}
\left\Vert \frac{\partial \tilde f_{m}}{\partial w}(W,x_i;\alpha)\right\Vert_F \left\Vert \frac{\partial \tilde f_{m}}{\partial w}(W,x_j;\alpha)
-\frac{\partial \tilde f_{m}}{\partial w}(W(0),x_j;\alpha)\right\Vert_F\\
&\quad\quad+(\log m)^{2/\alpha}\left\Vert \frac{\partial \tilde f_{m}}{\partial w}(W(0),x_j;\alpha)\right\Vert_F\ \left\Vert \frac{\partial \tilde f_{m}}{\partial w}(W,x_i;\alpha)
-\frac{\partial \tilde f_{m}}{\partial w}(W(0),x_i;\alpha)\right\Vert_F\\
&\quad\leq 
(\log m)^{2/\alpha}\left(\left\Vert \frac{\partial \tilde f_m}{\partial w}(W(0),x_i;\alpha)\right\Vert_F
+\left\Vert\frac{\partial \tilde f_m}{\partial w}(W(0),x_i;\alpha)-
\frac{\partial \tilde f_m}{\partial w}(W,x_i;\alpha)\right\Vert_F
\right)\\
&\quad\quad\quad\times \left\Vert \frac{\partial \tilde f_{m}}{\partial w}(W,x_j;\alpha)
-\frac{\partial \tilde f_{m}}{\partial w}(W(0),x_j;\alpha)\right\Vert_F\\
&\quad\quad+(\log m)^{2/\alpha}\left\Vert \frac{\partial \tilde f_{m}}{\partial w}(W(0),x_j;\alpha)\right\Vert_F \left\Vert \frac{\partial \tilde f_{m}}{\partial w}(W,x_i;\alpha)
-\frac{\partial \tilde f_{m}}{\partial w}(W(0),x_i;\alpha)\right\Vert_F.
\end{align*}
For every $i=1,\dots,k$, 
\begin{align*}
    (\log m)^{2/\alpha}\left\Vert \frac{\partial \tilde f_{m}}{\partial w}(W(0),x_i;\alpha)\right\Vert_F^2
     &= \frac{1}{m^{2/\alpha}}\sum_{i=1}^m \langle \wz_i(0),x_i\rangle^2I(|\langle \wz_i(0),x_i\rangle |>0)\\
     &\leq \frac{1}{m^{2/\alpha}}\sum_{i=1}^m \langle \wz_i(0),x_i\rangle^2,
\end{align*}
which converges in distribution, as $m\rightarrow\infty$. Thus there exist $M>0$ and $m_0$ such that for every $m\geq m_0$ and every $i=1,\dots,k$,
$$
\P\left((\log m)^{1/\alpha}
\left\Vert \frac{\partial\tilde f_{m}}{\partial w}(W(0),x_i;\alpha)\right\Vert_F>M
\right)<\frac{\delta}{8k^2}.
$$
By Lemma \ref{lem_stab3}, for $m$ sufficiently large, with probability at least $1-\delta/(4k^2)$
$$
(\log m)^{1/\alpha}\left(\left\Vert \frac{\partial \tilde f_m}{\partial w}(W(0),x_i;\alpha)\right\Vert_F
+ \left\Vert\frac{\partial \tilde f_m}{\partial w}(W(0),x_i;\alpha)-
\frac{\partial \tilde f_m}{\partial w}(W,x_i;\alpha)\right\Vert_F\right)<2M
$$
whenever $\Vert W-W(0)\Vert_F<(\log m)^{2/\alpha}$. 
Lemma \ref{lem_stab3} also implies that, for every $\gamma\in (0,1)$, and $i=1,\dots,k$, with probability at least $1-\delta/(8k^2)$ 
$$
(\log m)^{1/\alpha}\left\Vert \frac{\partial \tilde f_{m}}{\partial w}(W,x_i;\alpha)
-\frac{\partial \tilde f_{m}}{\partial w}(W(0),x_i;\alpha)\right\Vert_F<\frac{\lambda}{4Mk^2}m^{-\gamma/\alpha}
$$
whenever $\Vert W-W(0)\Vert_F^2<(\log m)^{4/\alpha}$, provided $m$ is sufficiently large,. 
Thus, with probability at least $1-\delta$, if $m$ is sufficiently large
$$
\max_{i,j}|\tilde{H}_{m}^{(2)}(W,X)[i,j]-\tilde{H}_{m}^{(2)}(W(0),X)[i,j]|<
\frac{\lambda}{k^2}m^{-\gamma/\alpha},
$$
whenever $\Vert W-W(0)\Vert_F<(\log m)^{2/\alpha}$.
Thus
\begin{align*}
&\Vert \tilde{H}_{m}^{(2)}(W,X)-\tilde{H}_{m}^{(2)}(W(0),X)\Vert_2\\
&\quad\leq \Vert \tilde{H}_{m}^{(2)}(W,X)-\tilde{H}_{m}^{(2)}(W(0),X)\Vert_F<\lambda m^{-\gamma/\alpha}<\frac{\lambda}{2},
\end{align*}
whenever $\Vert W-W(0)\Vert_F<(\log m)^{2/\alpha}$, provided $m$ is sufficiently large. The last inequality and Lemma \ref{eig2} imply that, with probability at least $1-\delta$, if $m$ is sufficiently large, then
$$
\Vert \tilde{H}_{m}^{(2)}(W,X)\Vert_2>\lambda/2,
$$
for every $W$ such that $\Vert W-W(0)\Vert_F<(\log m)^{2/\alpha}$.
Since $\tilde{H}_{m}(W,X)$ is the sum of two positive semi-definite matrices $\tilde{H}_{m}^{(1)}(W,X)$ and $\tilde{H}_{m}^{(2)}(W,X)$, then 
$$
\Vert \tilde{H}_{m}(W,X)\Vert_2\geq 
\Vert \tilde{H}_{m}^{(2)}(W,X)\Vert_2>\lambda/2,
$$
for every $W$ such that $\Vert W-W(0)\Vert_F<(\log m)^{2/\alpha}$, if $m$ is sufficiently large.
\end{proof}

\begin{lem}\label{lem_stab5}
For every $\delta>0$ the following property holds true, for $m$ sufficiently large, with probabillity at least $1-\delta$: there exists $M>0$ such that
	$$(\log m)^{1/\alpha}\left\Vert \frac{\partial \tilde f_{m}}{\partial \wz}(W,x_j;\alpha)-
	\frac{\partial \tilde f_{m}}{\partial \wz}(W(0),x_j;\alpha)\right\Vert_F<M,$$
	for every $j=1,\dots,k$, and for every $W$ such that $\Vert W-W(0)\Vert_F\leq (\log m)^{2/\alpha}$.
\end{lem}
\begin{proof}
Let us define $r_i=|\langle \wz_i-\wz_i(0),x_j\rangle|$ for $i=1,\ldots,m$. Now, since $\Vert x_j\Vert=1$ by assumption, for $j=1,\ldots,k$, then we can write
$$
\sum_i r_i^2\leq \Vert x_j\Vert^2\cdot \Vert \wz_i-\wz(0)\Vert_F^2\leq \Vert W-W(0)\Vert_F^2\leq (\log m)^{4/\alpha}.
$$
It holds
\begin{align*}
	&(\log m)^{2/\alpha}\left\Vert \frac{\partial \tilde f_{m}}{\partial \wz}(W,x_j;\alpha)-
	\frac{\partial \tilde f_{m}}{\partial \wz}(W(0),x_j;\alpha)\right\Vert_F^2\\
	&\quad\leq \frac{1}{m^{2/\alpha}}
	\sum_{i=1}^m \left(
	w_i I (\langle \wz_i,x_j\rangle >0)- w_i(0) I (\langle \wz_i(0),x_j\rangle>0)\right)^2\\
	&\quad\leq \frac{2}{m^{2/\alpha}}
	\sum_{i=1}^m 
	(w_i -w_i(0))^2I(\langle \wz_i,x_j\rangle >0)\\
	&\quad\quad+\frac{2}{m^{2/\alpha}}
	\sum_{i=1}^m w_i(0)^2 \vert
	I(\langle \wz_i,x_j\rangle >0)- I (\langle \wz_i(0),x_j\rangle>0)|.
\end{align*}
We will bound the two terms separately. First,
\begin{align*}
	&\frac{2}{m^{2/\alpha}}
	\sum_{i=1}^m 
	(w_i -w_i(0))^2I (\langle \wz_i,x_j\rangle >0)\\ 
	&\quad\leq \frac{1}{m^{2/\alpha}}
	\sum_{i=1}^m 
	(w_i -w_i(0))^2\\
	&\quad\leq \frac{2}{m^{2(1-\gamma)/\alpha}} \Vert w-w(0)\Vert_F^2\\
	&\quad\leq \frac{2}{m^{2/\alpha}}(\log m)^{4/\alpha}< \frac c 4 m^{-2\gamma/\alpha},
\end{align*}
if $m$ is sufficiently large. To bound the second term, we can write that
\begin{align*}
	&\frac{2}{m^{2/\alpha}}
	\sum_{i=1}^m w_i(0)^2 \vert
	I(\langle \wz_i,x_j\rangle >0)- I (\langle \wz_i(0),x_j\rangle>0)|\\
	&\quad\leq \frac{2}{m^{2/\alpha}}
	\sum_{i=1}^m w_i(0)^2,
\end{align*}
which converges in distribution to a stable random variable, as $m\rightarrow\infty$. Hence there exists $M_1$ such that, with probability at least $1-\delta/4$,
$$
\frac{2}{m^{2/\alpha}}
\sum_{i=1}^m 
(w_i -w_i(0))^2I(\langle \wz_i,x_j\rangle >0)<\frac{M_1^2}{2k^2}
$$
and
$$
\frac{2}{m^{2/\alpha}}
\sum_{i=1}^m w_i(0)^2 \vert
I(\langle \wz_i,x_j\rangle >0)- I (\langle \wz_i(0),x_j\rangle>0)|<\frac{M_1^2}{2k^2},
$$
for $m$ sufficiently large, which entail
$$(\log m)^{1/\alpha}\left\Vert \frac{\partial \tilde f_{m}}{\partial \wz}(W,x_j;\alpha)-
\frac{\partial \tilde f_{m}}{\partial \wz}(W(0)(\omega),x_j;\alpha)\right\Vert_F<\frac{M_1}{k}.$$
On the other hand,  there exist $N_3\in \mathcal F$ and $M_2$ with  $P(N_3)>1-\delta/4$ such that, for every $\omega\in N_3$ and for $m$ sufficiently large,
$$
\Vert \tilde f_{m}(W(0)(\omega),X;\alpha)-Y\Vert_F<M_2,
$$
and
$$
\max_{1\leq i\leq k}\left\Vert \frac{\partial}{\partial W}\tilde f_{m}(W(0)(\omega),x_i;\alpha)\right\Vert_F<M_2 (\log m)^{-1/\alpha}.
$$
The above inequalities follow from the convergence in distribution of $\tilde f_{m}(W(0),x_i;\alpha)$  and of 
$$(\log m)^{2/\alpha}\left\Vert \frac{\partial}{\partial W}\tilde f_{m}(W(0),x_i;\alpha)\right\Vert_F^2=
\tilde H(W(0),X;\alpha)[i,i]\quad (i=1,\dots,k),$$ 
as $m\rightarrow\infty$.
\end{proof}

\begin{lem}\label{lem_stab6}
	Let $\gamma\in (0,1)$ and $c>0$ be fixed numbers. For every $\delta>0$ the following property holds true, for $m$ sufficiently large, with probability at least $1-\delta$:
	$$\Vert W(t)-W(0)\Vert_F < (\log m)^{2/\alpha}.
	$$
	if 
	$$
	(\log m)^{2/\alpha}\left\Vert \frac{\partial \tilde f_{m}}{\partial w}(W(s),x_j;\alpha)-\frac{\partial \tilde f_{m}}{\partial w }(W(0),x_j;\alpha)\right\Vert_F^2\leq c m^{-2\gamma/\alpha}
	$$
	for every NN's input $x_{j}$, with $j=1,\ldots,k$, and for every $s\leq t$.
\end{lem}
\begin{proof}
		By Lemmas \ref{lem_stab3} and \ref{lem_stab4}, there exists $N_1\in\mathcal F$ with probability at least $1-\delta/2$ such that, for every $\omega\in N_1$,
	$$
	(\log m)^{2/\alpha}\left\Vert \frac{\partial \tilde f_{m}}{\partial w}(W,x_j;\alpha)-\frac{\partial \tilde f_{m}}{\partial w}(W(0)(\omega),x_j;\alpha)\right\Vert_F^2 < c m^{-2\gamma/\alpha},
	$$
	for arbitrarily fixed $c>$ and $\gamma\in (0,1/2)$, and 
	$$
	\lambda_{\text{min}}(\tilde{H}_{m}(W,X))>\frac{\lambda}{2},
	$$
	for some $\lambda>0$, for every $W$ such that $\Vert W-W(0)(\omega)\Vert_F\leq (\log m)^{2/\alpha}$ and every $j=1,\dots,k$, provided $m$ is sufficiently large. Moreover, by Lemma \ref{lem_stab5}, there exist, for $m$ sufficiently large, $M_1>0$ and $N_2$ with $\mathbb P(N_2)>1-\delta$, such that
	$$(\log m)^{1/\alpha}\left\Vert \frac{\partial \tilde f_{m}}{\partial \wz}(W,x_j;\alpha)-
	\frac{\partial \tilde f_{m}}{\partial \wz}(W(0)(\omega),x_j;\alpha)\right\Vert_F<\frac{M_1}{k},$$
	for every $j=1,\dots,k$, and for every $W$ such that $\Vert W-W(0)(\omega)\Vert_F\leq (\log m)^{2/\alpha}$.
	We will prove, by contradiction, that for every $\omega\in N_1\cap N_2\cap N_3$, $\Vert W(t)-W(0)\Vert_F<(\log m)^{2/\alpha}$ for every $t>0$. In the following we will write $W(s)$ in the place of  $W(s)(\omega)$ and always assume that $\omega$ belongs to $N_1\cap N_2\cap N_3$.\\
Suppose that there exists $t$ such that $\Vert W(t)-W(0)\Vert_F\geq (\log m)^{2/\alpha}$, and let
$$
t_0=\mathrm{argmin}_{t\geq 0}\{t:\Vert W(t)-W(0)\Vert_F\geq (\log m)^{2/\alpha}\}.
$$
Since $\Vert W(s)-W(0)\Vert_F\leq (\log m)^{2/\alpha}$ for every $s\leq t_0$, then, for every $s\leq t_0$,
\begin{align*}
	&\lambda_{\text{min}}(\tilde{H}_{m}(W(s),X))>\frac{\lambda}{2},
\\&
\left\Vert \frac{\partial \tilde f_{m}}{\partial w}(W(s),x_j;\alpha)-\frac{\partial \tilde f_{m}}{\partial w}(W(0),x_j;\alpha)\right\Vert_F < c m^{-\gamma/\alpha}(\log m)^{-1/\alpha},
\\&\left\Vert \frac{\partial \tilde f_{m}}{\partial \wz}(W(s),x_j;\alpha)-
\frac{\partial \tilde f_{m}}{\partial \wz}(W(0)(\omega),x_j;\alpha)\right\Vert_F<\frac{M_1}{k}(\log m)^{-1/\alpha}\quad(j=1,\dots,k),\\&
\Vert \tilde f_{m}(W(0)(\omega),X;\alpha)-Y\Vert_F<M_2,
\\&
\max_{1\leq i\leq k}\left\Vert \frac{\partial}{\partial W}\tilde f_{m}(W(0)(\omega),x_i;\alpha)\right\Vert_F<M_2(\log m)^{-1/\alpha}.
\end{align*}
Let us now consider the gradient descent dynamic, with continuous learning rate $\eta=(\log m)^{2/\alpha}$:
\begin{align*}
	\frac{\ddr W(s)}{\ddr s}&=-(\log m)^{2/\alpha}\nabla_{W}\frac{1}{2}\sum_{i=1}^{k}
	\left(\tilde{f}_{m}(W(s),x_{i};\alpha)-y_{i}\right)^2\\
	&=-(\log m)^{2/\alpha}\sum_{i=1}^{k}
	\left( \tilde{f}_{m}(W(s),x_{i})- y_i\right)\frac{\partial \tilde f_{m}}{\partial W}(W(s),x_i;\alpha).
\end{align*}
This expression allows to write
\begin{align*}
   &\left \Vert W(t_0)-W(0)\right\Vert_F\\
   &\quad\leq \left\Vert \int_0^{t_0} \frac{\ddr}{\ddr s}W(s)\ddr s\right\Vert_F\\
    &\quad\leq(\log m)^{2/\alpha} \left\Vert \int_0^{t_0} \sum_{i=1}^k(\tilde f_{m}(W(s),x_i;\alpha)-y_i)\frac{\partial \tilde f_{m}}{\partial W}(W(s),x_i;\alpha)\ddr s\right\Vert_F\\
    &\quad\leq (\log m)^{2/\alpha}\max_{0\leq s\leq t_0}
    \sum_{i=1}^k \left\Vert 
    \frac{\partial\tilde  f_{m}}{\partial W}(W(s),x_i;\alpha)\right\Vert_F
    \int_0^{t_0}\Vert\tilde f_{m}(W(s),X;\alpha)-Y\Vert ds.
\end{align*}
To bound the term $\Vert\tilde f_{m}(W(s),X;\alpha)-Y\Vert$ we will exploit the dynamics of the NN output
\begin{align*}
	\frac{\ddr \tilde{f}_{m}(W(s),X;\alpha)}{\ddr s}&=
	 \frac{\partial \tilde f_{m}}{\partial W}(W(s),X;\alpha)\frac{\ddr W^T(s)}{\ddr s}\\
	&=-(\log m)^{2/\alpha} (\tilde{f}_{m}(W(s),X;\alpha)-Y)H_{m}(W(s),X)\\
	&=-(\tilde{f}_{m}(W(s),X;\alpha)-Y)\tilde H_{m}(W(s),X),
\end{align*}
that gives
\begin {align*}
\frac{\ddr }{\ddr s}\Vert \tilde{f}_{m}(W(s),X;\alpha)-Y\Vert_2^2
&=-2\left( \tilde{f}_{m}(W(s),X;\alpha)-Y \right)\tilde  H_{m}(W(s),X)
\left( \tilde{f}_{m}(W(s),X;\alpha)-Y \right)^T.
\end{align*}
Since $\lambda_{\text{min}}(\tilde H_m(W(s),X))>\lambda/2$ for every $s\leq t_0$, then
\begin{align*}
\frac{\ddr }{\ddr s}\Vert \tilde{f}_{m}(W(s),X;\alpha)-Y\Vert_2^2&\leq -\lambda \Vert \tilde{f}_{m}(W(s),X;\alpha)-Y\Vert_2^2,
\end{align*}
which implies that
$$
\frac{\ddr }{\ddr s}\left(\exp(\lambda s)\Vert \tilde{f}_{m}(W(s),X;\alpha)-Y\Vert_2^2\right)\leq 0.
$$
It follows that $\exp(\lambda s)\Vert \tilde{f}_{m}(W(s),X;\alpha)-Y\Vert_2^2$ is a decreasing function of $s$, and therefore
$$
\Vert \tilde{f}_{m}(W(s),X;\alpha)-Y\Vert_2\leq \exp(-\lambda /2)\Vert \tilde{f}_{m}(W(0),X;\alpha)-Y\Vert_2,
$$
for every $s\leq t_0$.
Substituting in the integral, we can write that
\begin{align*}
	&\left \Vert W(t_0)-W(0)\right\Vert_F\\
	&\quad\leq (\log m)^{2/\alpha} \max_{0\leq s\leq  t_0}
	\sum_{i=1}^k \left\Vert 
	\frac{\partial \tilde f_{m}}{\partial W}(W(s),x_i;\alpha)\right\Vert_F
	\int_0^{t_0} \exp(-\lambda s/2)\ddr s\cdot \Vert \tilde f_{m}(W(0),X;\alpha)-Y\Vert\\
	&\quad \leq \frac{2(\log m)^{2/\alpha}}{\lambda }\max_{0\leq s\leq  t_0}
	\sum_{i=1}^k \left(\left\Vert 
	\frac{\partial \tilde f_{m}}{\partial W}(W(0),x_i;\alpha)\right\Vert_F
	+\left\Vert 
	\frac{\partial \tilde f_{m}}{\partial W}(W(s),x_i;\alpha)-
	\frac{\partial \tilde f_{m}}{\partial W}(W(0),x_i;\alpha)\right\Vert_F\right)\\
	&\quad\quad\times \Vert\tilde  f_{m}(W(0),X;\alpha)-Y\Vert\\
	&\quad \leq \frac{2(\log m)^{2/\alpha}}{\lambda }\max_{0\leq s\leq  t_0}
	\sum_{i=1}^k \left(\left\Vert 
	\frac{\partial \tilde f_{m}}{\partial W}(W(0),x_i;\alpha)\right\Vert_F
	+\left\Vert 
	\frac{\partial \tilde f_{m}}{\partial \wz}(W(s),x_i;\alpha)-
	\frac{\partial \tilde f_{m}}{\partial \wz}(W(0),x_i;\alpha)\right\Vert_F\right.\\
	&\quad\quad\quad +\left.\left\Vert 
	\frac{\partial \tilde f_{m}}{\partial w}(W(s),x_i;\alpha)-
	\frac{\partial \tilde f_{m}}{\partial w}(W(0),x_i;\alpha)\right\Vert_F\right) \Vert\tilde  f_{m}(W(0),X;\alpha)-Y\Vert\\
	&\leq\quad \frac{2(\log m)^{1/\alpha}}{\lambda } 
	\left(M_2+M_1+kcm^{-\gamma/\alpha}
	\right)M_2,
\end{align*}
which, for $m$ large, contradicts $\Vert W(t_0)-W(0)\Vert_F\geq (\log m)^{2/\alpha}$ .
\end{proof}

\bigskip

\begin{proof}[\bf Proof of Theorem \ref{teo_ntk0}]
	Let $m\in\mathbb N$ and $N\in\mathcal F$ be such that $\P(N)>1-\delta$ and the properties mentioned in Lemma \ref{lem_stab3}, Lemma \ref{lem_stab4}, Lemma \ref{lem_stab5} and Lemma \ref{lem_stab6} hold true for every $\omega\in N$. Therefore, by means of Lemma \ref{lem_stab3} and of Lemma \ref{lem_stab4}, it is sufficient to show that
	$$
	\Vert W(t)-W(0)\Vert_F^2(\omega)<(\log m)^{2/\alpha}
	$$
	for every $t>0$ and $\omega\in N$. By contradiction, suppose that there exists, for some $\omega\in N$,  $t_0(\omega)$ finite with
	$$
	t_0(\omega):=\inf_{t\geq 0}\{t:\Vert W(t)-W(0)\Vert_F(\omega)\geq  (\log m)^{2/\alpha}\}.
	$$
	Since $W(t)(\omega)$ is a continuous function of $t$, then $\Vert W(t_0(\omega))-W(0)\Vert_F^2(\omega)=(\log m)^{2/\alpha}$. Then, by Lemma \ref{lem_stab3},
	$$
	(\log m)^{2/\alpha}\left\Vert \frac{\partial \tilde f_{m}}{\partial W}(W(s),x_j;\alpha)-\frac{\partial \tilde f_{m}}{\partial W}(W(0),x_j;\alpha)\right\Vert_F^2(\omega) < c m^{-2\gamma/\alpha},
	$$
	for every  $s\leq t_0$ and every $j$. Therefore, by Lemma \ref{lem_stab6} it holds true that $\Vert W(t_0(\omega))-W(0)\Vert_F(\omega)<(\log m)^{2/\alpha}$, which contradicts the definition of $t_0$.
\end{proof}

\section{}\label{appendixid}
The distribution of a random vector $\xi$ is said to be infinitely divisible if, for every $n$, there exist some i.i.d. random vectors $\xi_{n1},\dots,\xi_{nn}$ such that $\sum_k\xi_{nk}\stackrel{d}{=}\xi$. 
A $k$-dimensional random vector $\xi$ is infinitely divisible if and only if its characteristic function admits the representation $e^{\psi( u)}$, where
\begin{equation}
	\label{eq:id}
	\psi( u)=i u^T b-\frac{1}{2} u^T a u+
	\int\left(e^{i u^T  x}-1-i u^T xI(|| x||\leq 1)\right)\nu(d x)
\end{equation}
where $\nu$ is a measure on $\mathbb R^k\setminus \{ 0\}$ satisfying $\int (|| x||^2\wedge 1)\nu(d x)<\infty$, $ a$ is a $k\times k$ positive semi-definite, symmetric matrix and $ b$ is a vector. The measure $\nu$ is called the L\'evy measure of $\xi$ and $( a, b,\nu)$ are called the characteristics of the infinitely divisible distribution. We will write $\xi\sim i.d.( a, b,\nu)$. Other kinds of truncation can be used for the term $i u^T  x$. This affects only the vector of centering constants $ b$.
An i.i.d. array of random vectors is a collection of random vectors $\{\xi_{nj},j\leq m_n,n\geq 1\}$ such that, for every $n$, $\xi_{n1},\dots,\xi_{nm_n}$ are i.i.d.
The class of infinitely divisible distributions coincides with the class of limits of sums of i.i.d. arrays \citep[Theorem 13.12]{Kal(02)}.

To state a general criterion of convergence, we first introduce some notations. 
Let $\xi\sim i.d.( a, b,\nu)$. Define, for each $h>0$,
\[
 a^{(h)}= a+\int_{|| x||<h} x x^T\nu(d x),
\]
\[
 b^{(h)}= b-\int_{h<|| x||\leq 1} x\nu(d x),
\]
where $\int_{h<|| x||\leq 1}=-\int_{1<|| x||\leq h}$ if $h>1$. Denote by $\stackrel{v}{\rightarrow}$ vague convergence, that is convergence of measures with respect to the topology induced by bounded, measurable functions with compact support. Moreover, let $\overline{\mathbb R^k}$ be the one-point compactification of $\mathbb R^k$.
The following criterion for convergence holds \citep[Corollary 13.16]{Kal(02)}.

\begin{theorem}
	\label{th:clt}
	Consider in $\mathbb R^k$ an i.i.d. array $(\xi_{nj})_{j=1,\dots,m_n,n\geq 1}$ and let $\xi$ be $i.d.( a, b,\nu)$.  Let $h>0$ be such that $\nu(||x||=h)=0$. Then $\sum_j\xi_{nj}\stackrel{d}{\rightarrow} \xi$ if and only if the following conditions hold:
	\begin{description}
		\item (i) $m_n\P\left(\xi_{n1}\in \cdot\right)\stackrel{v}{\rightarrow} \nu(\cdot)$ on $\overline{\mathbb R^k}\setminus \{ 0\}$
		\item (ii) $m_n\E(\xi_{n1}\xi_{n1}^TI(||\xi_{n1}||<h))\rightarrow  a^{(h)}$
		\item (iii) $m_n\E(\xi_{n1}I(||\xi_{n1}||<h))\rightarrow  b^{(h)}$
	\end{description}
\end{theorem}
Inside the class of infinitely divisible distribution, we can distinguish the subclass of stable distributions.
A $k$-dimensional random vector $\xi$ has stable distribution if, for every independent random vectors $\xi_1$ and $\xi_2$ with $\xi_1\stackrel{d}{=}\xi_2\stackrel{d}{=}\xi$ and every $a,b\in\mathbb R$, there exists $c\in\mathbb R$ and $ d\in\mathbb R^k$ such that
$a\xi_1+b\xi_2\stackrel{d}{=}c\xi+ d$. This is equivalent to the condition: for every $n\geq 1$,
\begin{equation}
	\label{eq:stable}
	\xi_1+\dots+\xi_n\stackrel{d}{=}n^{1/\alpha}\xi+ d_n
\end{equation}
where $\alpha\in (0,2]$, $\xi_1,\dots,\xi_n$ are i.i.d. copies of $\xi$ and $ d_n$ is a vector. The random vector $\xi$ is said to be strictly stable if \eqref{eq:stable} holds with $ d_n= 0$. A stable vector $\xi$ is strictly stable if and only if all its components are strictly stable. The coefficient $\alpha$ is called the index of stability of $\xi$ and the law of $\xi$ is called $\alpha$-stable. A stable vector $\xi$ is symmetric stable if $\mathbb{P}(\xi\in A)=\mathbb{P}(-\xi\in A)$ for every Borel set $A$. A symmetric stable vector is strictly stable. 
The class of stable distributions coincides with the class of limit laws of sequences $((\sum_{k=1}^nX_k-b_n)/a_n)$, where $(X_n)$ are i.i.d. random variables.

A stable distribution is infinitely divisible. Thus its characteristic function admits the L\'evy representation \eqref{eq:id}. If $\alpha=2$, then the L\'evy measure is the null measure and, therefore, the stable distribution coincides with the multivariate normal distribution with covariance matrix $ a$ and mean vector $ b$. 
If $\alpha<2$, then $ a=0$ (the zero matrix) and the $\alpha$-stability implies that there exists a measure $\sigma$ on the unit sphere $\mathbb S^{k-1}$ such that $\nu(d x)=r^{-(\alpha+1)}dr\sigma(d s)$, where $r=|| x||$ and $ s= x/|| x||$. 
Substituting in \eqref{eq:id}, we obtain
\[
\psi( u)=i u^T b+\int_S\int_0^\infty \left(e^{ir u^T s}-1-
ir u^T sI(r\leq 1)\right)\frac{1}{r^{1+\alpha}}dr \sigma(d s)
\]
For  $\alpha<1$, the centering $ir u^T sI(r\leq 1)$ is not needed, since the function (of $r$) is integrable, and we can write 
\[\psi( u)
=i u^T b'+\int_S\int_0^\infty \left(e^{ir u^T s}-1\right)\frac{1}{r^{1+\alpha}}dr \sigma(d s),
\]
for some vector $ b'$. After evaluating the inner integrals as in \citet[Example XVII.3]{Fel(68)}, we obtain
\[
\psi( u)=i u^T b'-
\int_S | u^T s|^\alpha
\Gamma(1-\alpha)\left(\cos(\pi\alpha/2)-i \,\text{sign}( u^T s)\sin(\pi\alpha/2)\right)\sigma(d s)
\]
\[
=i u^T b'-
\int_S  | u^T s|^\alpha\left(1-i \,\text{sign}( u^T s)\tan(\pi\alpha/2)\right)\Gamma(1-\alpha)\cos(\pi\alpha/2)\sigma(d s).
\]
For  $\alpha>1$, using the centering $ir u^T s$, we can write  
\[\psi( u)
=i u^T b''+\int_S\int_0^\infty \left(e^{ir u^T s}-1-ir u^T s\right)\frac{1}{r^{1+\alpha}}dr \sigma(d s),
\]
for some $ b''$. After evaluating the inner integrals as in \citet[Example XVII.3]{Fel(68)}, we obtain
\[
\psi( u)=i u^T b''+
\int_S | u^T s|^\alpha
\frac{\Gamma(2-\alpha)}{\alpha-1}\left(\cos(\pi\alpha/2)-i \,\text{sign}( u^T s)\sin(\pi\alpha/2)\right)\sigma(d s)
\]
\[
=i u^T b''-
\int_S  | u^T s^\alpha\left(1-i \,\text{sign}( u^T s)\tan(\pi\alpha/2)\right)\frac{\Gamma(2-\alpha)}{1-\alpha}\cos(\pi\alpha/2)\sigma(d s).
\]
Since, for $\alpha<1$, $\Gamma(2-\alpha)=(1-\alpha)\Gamma(1-\alpha)$, we can encompass the above results in one equation, and write, for $\alpha\neq 1$,
\[
\psi( u)=i u^T b'''-
\int_S | u^T s|^\alpha\left(1-i \,\text{sign}( u^T s)\tan(\pi\alpha/2)\right)\frac{\Gamma(2-\alpha)}{1-\alpha}\cos(\pi\alpha/2)\sigma(d s),
\]
for some $ b'''$.
Finally, for $\alpha=1$, using the centering $ir\sin r  u^T s $, we can write
\[
\psi( u)=i u^T b''''+\int_S\int_0^\infty \left(e^{ir u^T s}-1-
ir\sin r  u^T s \right)\frac{1}{r^{2}}dr \sigma(d s),
\]
for some $ b''''$. Evaluating the inner integral as in \citet[Example XVII.3]{Fel(68)}, we obtain
\[
\psi( u)=i u^T b''''-\int_S
| u^T s|\left(\frac{\pi}{2}+i \text{sign}( u^T s)\log | u^T s|\right)
\sigma(d s)
\]
\[
=i u^T b''''-\int_S
| u^T s|\left(1+i\frac{2}{\pi} \text{sign}( u^T s)\log | u^T s|\right)\frac{\pi}{2}
\sigma(d s).
\]
Considering the spectral representation $e^{\psi( u)}$ of the multivariate stable characteristic function
\[
\psi( u)=
\left\{
\begin{array}{ll}
	-\int_S
	| u^T s|^\alpha \left(1-i\;\text{sign}( u^T s)\tan(\pi\alpha/2)\right) \Gamma(ds)+i u^T\mu^{(0)}
	&\alpha\neq 1\\
	\\
	-\int_S
	| u^T s| \left(1+i\frac{2}{\pi}\;\text{sign}( u^T s)\log| u^T s|\right) \Gamma(ds)+i u^T\mu^{(0)}
	&
	\alpha=1,
\end{array}
\right.
\]
we can establish the following relationship between the L\'evy measure $\nu$ and the spectral measure $\Gamma$:
\[
\nu(dx)=C_\alpha \frac{1}{r^{\alpha+1}}\Gamma(ds),
\]
where $r=||x||$, $s=x/||x||$ and
\[
C_\alpha=\left\{
\begin{array}{ll}
	\displaystyle{\frac{1-\alpha}{\Gamma(2-\alpha)\cos(\pi\alpha/2)}}&\alpha\neq 1\\
	\\
	2/\pi&\alpha=1
\end{array}\right.
\]
A Stable random vector $\xi$ is strictly stable if and only if 
\[
\left\{
\begin{array}{ll}
	\mu^{(0)}=  0&\alpha\neq 1\\
	\int_S s_j\Gamma(ds)=0 \mbox{ for every j }& \alpha=1.
\end{array}
\right.
\]
(see e.g. \citet[Theorem 2.4.1]{Sam(94)}).
By Theorem \ref{th:clt}, the spectral measure $\Gamma$ of a symmetric stable random vector $\xi$ satisfies
\begin{equation}
	\label{eq:levy}
	\lim_{n\rightarrow \infty}n\mathbb{P}\left(||\xi||>n^{1/\alpha} x,\frac{\xi}{||\xi||}\in A\right)=C_\alpha x^{-\alpha}\Gamma(A)
\end{equation}
for every  Borel set $A$ of $S$ such that $\Gamma (\partial A)=0$. Moreover, the distribution of a random vector $\xi$ belongs to the domain of attraction of the $\text{St}_{k}(\alpha,\Gamma)$ distribution, with $\alpha\in (0,2)$ and $\Gamma$ simmetric finite measure on $\mathbb S^{k-1}$, if and only if \eqref{eq:levy} holds (see e.g.  \citet[ Theorem 4.3]{Davydov2008}).


\section*{Acknowledgement}

Stefano Favaro is grateful to Professor Lorenzo Rosasco for some stimulating conversations on the neural tangent kernel and for valuable suggestions. Stefano Favaro received funding from the European Research Council (ERC) under the European Union's Horizon 2020 research and innovation programme under grant agreement No 817257. Stefano Favaro gratefully acknowledges the financial support from the Italian Ministry of Education, University and Research (MIUR), ``Dipartimenti di Eccellenza" grant agreement 2018-2022.



\end{document}